\documentclass[journal]{IEEEtran}
\usepackage{amsmath}
\usepackage{bbding}
\usepackage{bbm}
\usepackage{amssymb}
\usepackage{booktabs}
\usepackage{makecell}
\usepackage{graphicx}
\usepackage{subfigure}
\usepackage{bmpsize}
\usepackage{epstopdf}
\usepackage{cite}
\usepackage[ruled,vlined]{algorithm2e}
\usepackage{colortbl}
\usepackage[utf8]{inputenc}
\usepackage{pifont}

\usepackage[pagebackref=false,breaklinks=true,colorlinks,bookmarks=false]{hyperref}
\usepackage{microtype}
\usepackage{cleveref}
\usepackage{tabularray}
\usepackage{multirow}
\usepackage{multicol}
\usepackage{color}
\usepackage[table]{xcolor} 
\usepackage{colortbl}

\usepackage{amsthm}
\newtheorem{theorem}{Theorem}

\Crefname{figure}{Fig.}{Figs.}
\Crefname{equation}{Eq.}{Eqs.}
\Crefname{theorem}{Theorem}{Theorems}

\begin{document}
\title{CKD: Contrastive Knowledge Distillation from A Sample-Wise Perspective}

\author{
Wencheng~Zhu, Xin~Zhou, Pengfei~Zhu, Yu~Wang, and~Qinghua~Hu, \IEEEmembership{Senior~Member,~IEEE}
\thanks{Wencheng Zhu, Xin Zhou, Pengfei Zhu, Yu Wang, and Qinghua Hu are with the School of Artificial Intelligence, Tianjin University, Tianjin, 300372, China. Email: wenchengzhu@tju.edu.cn; zhouxinzzz@tju.edu.cn; zhupengfei@tju.edu.cn;
wangyu{\_}@tju.edu.cn; huqinghua@tju.edu.cn.
\emph{The corresponding author is Pengfei Zhu.}
}
}

\markboth{}%
{Shell \MakeLowercase{\textit{et al.}}: Bare Demo of IEEEtran.cls for IEEE Journals}

\maketitle

\begin{abstract}

In this paper, we propose a simple yet effective contrastive knowledge distillation framework that achieves sample-wise logit alignment while preserving semantic consistency. 
Conventional knowledge distillation approaches exhibit over-reliance on feature similarity per sample, which risks overfitting,  and contrastive approaches focus on inter-class discrimination at the expense of intra-sample semantic relationships.
Our approach transfers "dark knowledge" through teacher-student contrastive alignment at the sample level.
Specifically, our method first enforces intra-sample alignment by directly minimizing teacher-student logit discrepancies within individual samples.
Then, we utilize inter-sample contrasts to preserve semantic dissimilarities across samples.
By redefining positive pairs as aligned teacher-student logits from identical samples and negative pairs as cross-sample logit combinations, we reformulate these dual constraints into an InfoNCE loss framework, reducing computational complexity lower than sample squares while eliminating dependencies on temperature parameters and large batch sizes.
We conduct comprehensive experiments across three benchmark datasets, including the \textit{CIFAR-100}, \textit{ImageNet-1K}, and \textit{MS COCO} datasets, and experimental results clearly confirm the effectiveness of the proposed method on image classification, object detection, and instance segmentation tasks.


\end{abstract}


\IEEEpeerreviewmaketitle

\section{Introduction}

\IEEEPARstart{K}{nowledge} distillation (KD) \cite{tang2020understanding,gou2021knowledge} effectively compresses knowledge without significant performance degradation, making it possible to deploy complex models on resource-limited devices \cite{li2023object,wang2021knowledge,hao2023vanillakd}. 
Owing to its practical benefits, knowledge distillation has garnered much attention over the past decade \cite{cho2019efficacy,li2023curriculum}. Numerous knowledge distillation approaches have been proposed \cite{dong2023diswot,zhang2023pointdistiller} and introduced into a wide spectrum of tasks \cite{beyer2022knowledge,wang2023efficient}, such as image classification \cite{guo2023linkless,tu2022general}, object detection \cite{yang2022focal,shen2022fast}, and segmentation \cite{liu2019structured,shu2021channel,yang2022cross}. Despite considerable improvements achieved, existing knowledge distillation methods still struggle to determine what knowledge to distill and how to distill it \cite{stanton2021does, Zeyuan2023towards,ojha2024knowledge}. Here, we posit that the student model can well replicate the outputs of the teacher model when each sample is properly aligned.

\begin{figure}[t] 
\centering 
\includegraphics[width=1.0\columnwidth]{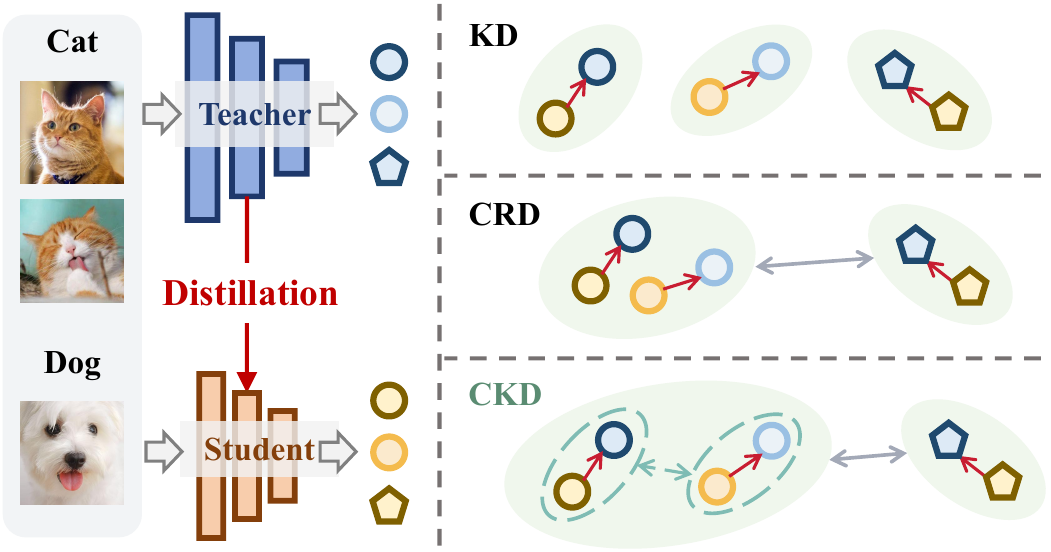} 
\vspace{-0.2cm}
\caption{
Comparison of three knowledge distillation approaches. Classic KD aligns feature similarities between paired teacher-student samples, while CRD leverages contrastive learning to align class-wise semantics. In contrast, the proposed CKD preserves sample-wise similarity for intra-sample alignment while capturing structural semantics for inter-sample contrast.
} 
\label{fig1} 
\end{figure}

Hinton \textit{et al.} \cite{hinton2015distilling} introduced seminal work on knowledge distillation that transfers knowledge by minimizing the Kullback-Leibler (KL) divergence between teacher and student logits.
Subsequent research has evolved along two main groups, including logit-based and feature-based approaches. 
Logit-based approaches mimic output teacher logits \cite{yang2023knowledge,li2023automated}. 
Representative methods are vanilla KD \cite{hinton2015distilling}, DML \cite{zhang2018deep}, DIST \cite{huang2022knowledge}, and DKD \cite{zhao2022decoupled}. 
Despite minimal computational and storage requirements \cite{li2022ckdf,li2021reskd,zhang2022latent}, these approaches generally underperform feature-based methods. 

Feature-based approaches align intermediate layer representations between teacher and student networks  \cite{lin2022knowledge,hao2023ofa}. 
Typical methods include FitNet \cite{romero2014fitnets}, OFD \cite{heo2019comprehensive}, and ReviewKD \cite{chen2021distilling}. 
These methods achieve significant performance, however, they often overfit due to rote memorization, achieving a low training loss but poor generalization. 
Subsequent methods like RKD \cite{park2019relational} address this limitation by modeling high-order feature correlations, whereas SP \cite{tung2019similarity} enhances robustness through pairwise similarity preservation in representation space. 
While these techniques alleviate over-fitting to some extent, their core focus on structural relationships may neglect the transfer of semantic knowledge.
Contrastive learning methods, such as CRD \cite{tian2020representation}, enhance structural representation distillation.
However, this approach encounters practical limitations: the reliance on large negative sample batches necessitates a large memory buffer, incurring high computational costs, especially for dense prediction tasks.

To address structure-preserving limitations, we propose a \textbf{C}ontrastive \textbf{K}nowledge \textbf{D}istillation approach, dubbed CKD, which unifies intra- and inter-sample distillation in a contrastive learning framework. \Cref{fig1} shows the main idea of the proposed approach compared to two representative approaches.
We can observe that our approach aligns same-sample logits while separating cross-sample logits.
Our method draws inspiration from the cross-modal alignment paradigm of CLIP, where we take knowledge distillation as a dual-modality alignment task between teacher and student embedding space with the fixed teacher networks. Unlike conventional approaches \cite{romero2014fitnets,chen2021distilling} that enforce feature mimicking, our framework conducts cross-modal alignment through contrastive training, and this paradigm enables the student to capture the teacher's semantic representations beyond feature consistency while leveraging the efficient sample utilization strategy of CLIP for distillation effectiveness \cite{phuong2019towards, wang2020understanding}. However, directly applying CLIP is problematic as the teacher model is fixed. Therefore, we propose sample-wise contrastive learning with dual constraints: the intra-sample distillation enhances logit similarities within individual samples, and the inter-sample distillation increases differences across samples.

The intra- and inter-sample alignment constraints can be formulated as a sample-wise contrastive learning problem. 
Unlike existing contrastive distillation methods \cite{tian2020representation,bai2020feature} 
that focus on class-wise semantic structures, the sample-wise framework has two key advantages. 
First, our method directly optimizes the similarity and dissimilarity of logits within individual sample rather than aligning teacher and student models at the class level.
Second, by eliminating the dependency on class annotations, the framework supports fully unsupervised distillation.
Notably, a key benefit of this sample-wise framework lies in its computational efficiency. 
Different from class-wise methods requiring large batch sizes to sample the entire negative space, our approach needs a minimal number of negative samples.
Additionally, we prioritize logit distillation over feature distillation for two reasons. On one hand, logits provide compact and high-level semantic representations that are suitable for contrastive learning. 
On the other hand, compared to features, the lower dimensionality of logits significantly simplifies negative sampling.
The framework remains model-agnostic and can be seamlessly integrated with feature distillation approaches.

Our contributions can be summarized into three aspects:
\begin{enumerate}

\item We propose a contrastive knowledge distillation framework with dual constraints, in which our method simultaneously conducts intra-sample and inter-sample distillation between teacher and student logits.
    
\item We cast these dual intra-sample and inter-sample constraints into a sample-wise contrastive formulation through customized positive and negative pairs, achieving effective and computationally efficient training. 
    
\item Experiments on benchmark datasets validate the superiority of our method. Without bells and whistles, our method outperforms vanilla KD by 1.4\% and 0.98\% on the CIFAR-100 and ImageNet-1K datasets, respectively.

\end{enumerate}

\section{Related Work}
In this section, we first outline recent progress in knowledge distillation \cite{menon2021statistical,lin2023supervised,li2024promptkd}. Then, we review related contrastive learning approaches.

\subsection{Knowledge Distillation}
Knowledge distillation is a widely adopted model compression technique \cite{song2022spot,menon2021statistical,yang2023online} that allows compact models to achieve competitive performance with their larger counterparts
\cite{aguilar2020knowledge, niu2022respecting,meng2023distillation}.
As a pioneer, Hinton \textit{et al.} \cite{hinton2015distilling} introduced the concept of knowledge distillation that transfers ``dark knowledge'' through minimization of the KL divergence between teacher and student logit distributions.
Over the past few years, a rich line of knowledge distillation methods has been proposed \cite{klingner2023x3kd, cho2023ambiguity, cui2023kd}. 
These methods come in two flavors: logit distillation and feature distillation \cite{miles2023mobilevos,li2023rethinking}. 

\subsubsection{Logit Distillation}
Early logit distillation approaches mainly focused on enhancing model generalization through regularization and optimization. 
However, research on early-stage logit distillation has remained limited due to its inferior performance.
Initial attempts like DML \cite{zhang2018deep} improved network generalization via mutual learning across model ensembles, while TAKD \cite{mirzadeh2020improved} bridged the teacher-student capacity gap through multi-step distillation with intermediate networks. 
Subsequent advances introduced feature representation enhancement, distributional alignment, and automated framework design. ICKD-C \cite{liu2021exploring} captured feature diversity through grid-level inter-channel correlations. 
GLD \cite{kim2021distilling}combined local spatial pooling for fine-grained knowledge extraction. 
DIST \cite{huang2022knowledge} aligned inter-class and intra-class probabilistic relations, respectively. 
DKD \cite{zhao2022decoupled} decoupled logits into a target class and all non-target classes to improve transfer flexibility. 
Auto-KD \cite{li2023automated} employed a Monte Carlo tree search for architecture optimization. 
CTKD \cite{li2023curriculum} dynamically adjusted task difficulty via learnable temperature scheduling. 
MLKD \cite{jin2023multi} proposed a multi-level logit distillation framework that distills instance-level, batch-level, and class-level information. 
LSKD \cite{sun2024logit} took temperature as the standard deviation of logit and applied logit standardization before the softmax prediction. 
Moreover, there are theoretical studies \cite{phuong2019towards,cheng2020explaining} that further elucidate the principles of distillation. 
Our method belongs to logit distillation that leverages contrastive teacher-student logits to enhance knowledge transfer.

\begin{figure*}[t] 
\centering 
\includegraphics[width=2.0\columnwidth]{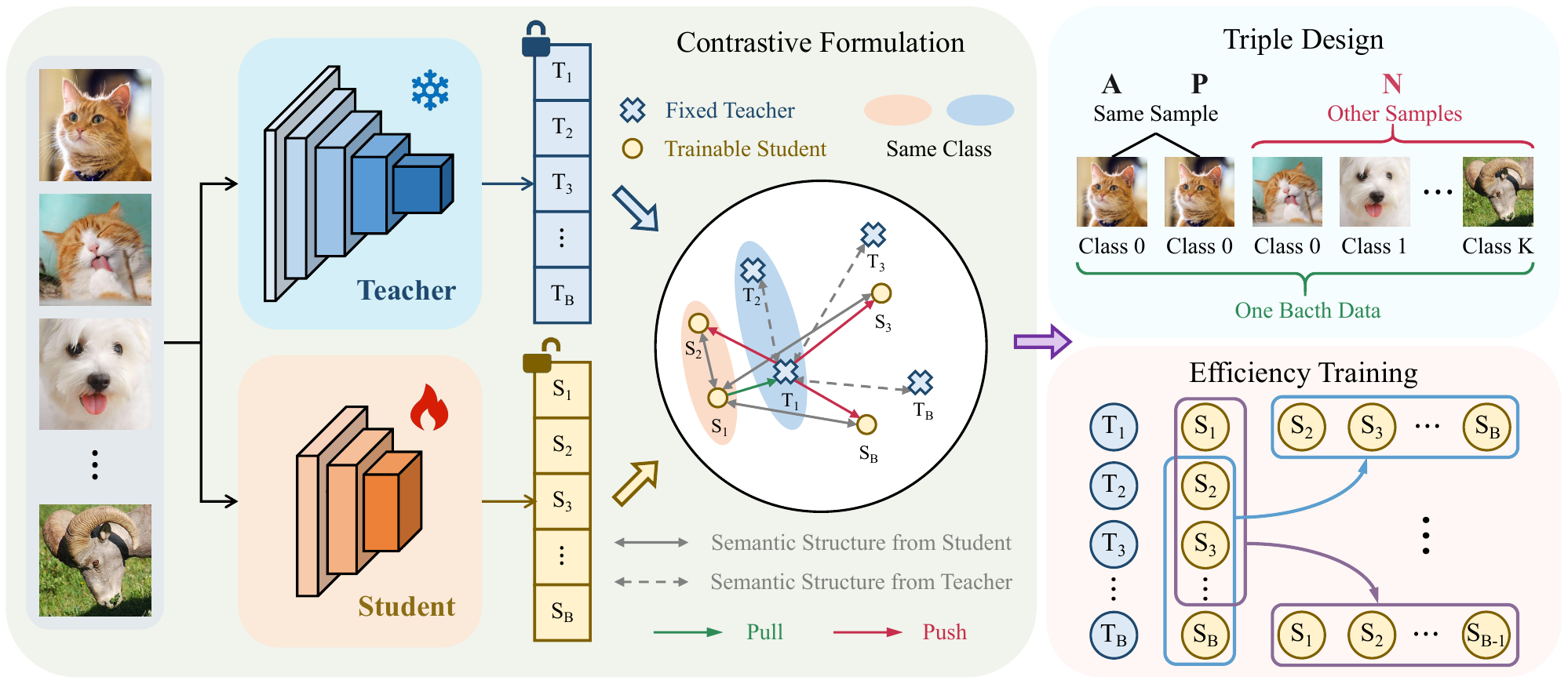} 
\vspace{-0.2cm}
\caption{
The overall architecture of the proposed Contrastive Knowledge Distillation framework. 
The framework incorporates newly designed triplets to optimize intra-sample feature similarity and cross-sample semantics simultaneously. For each input instance, the teacher and student logits form positive pairs, while other samples within the mini-batch serve as negative pairs. This formulation significantly enhances training efficiency through computationally efficient negative sample reuse, effectively addressing the memory constraints typically associated with contrastive learning approaches.
} 
\label{fig2} 
\end{figure*}

\subsubsection{Feature Distillation}
Feature distillation focuses on aligning intermediate representations of teacher and student models. 
FitNet \cite{romero2014fitnets} pioneered feature distillation by using hint layers to match student features to teacher activations, but its strict alignment often caused overfitting. 
Subsequent methods adopted different alignment strategies.
For example, AT \cite{komodakis2017paying} mimicked spatial attention maps of a teacher model to emphasize task-relevant regions. 
Differently, RKD \cite{park2019relational} preserved relational structures between samples by using distance-wise and angle-wise distillation objectives. 
IRG \cite{liu2019knowledge} incorporated both instance relationships and cross-layer feature transformations. 
CCKD \cite{peng2019correlation} captured multi-instance correlations for knowledge transfer. 
SP \cite{tung2019similarity} introduced a similarity-preserving loss to maintain the activation patterns of input pairs. Instead, OFD \cite{heo2019comprehensive} introduced a margin ReLU, an optimized distillation feature position, and a selective L2 distance for efficient network compression.
ReviewKD \cite{chen2021distilling} leveraged cross-layer feature aggregation to guide the learning of a student model. 
SimKD \cite{chen2022knowledge} reused the teacher's classifier and introduced an additional projector to align student features. 
CAT-KD \cite{guo2023class} enhanced the interpretability of knowledge distillation approaches by distilling class activation maps to highlight discriminative regions. 
While feature-based distillation generally outperforms logit-based methods in accuracy, it incurs significant computational overhead due to intermediate feature storage and alignment.

Existing contrastive knowledge distillation methods typically align teacher-student feature representations by enforcing intra-class similarity and inter-class separation through contrastive learning frameworks. 
For example, CRD \cite{tian2020representation} treats teacher features as anchors and leverages large-scale negative samples to push student features away from dissimilar classes.
CRCD \cite{zhu2021complementary}, which optimizes complementary relationships to reduce dependency on negatives. 
While these approaches demonstrate theoretical promise, there are several limitations, including heavy reliance on large-scale negative sample pairs, compromised scalability due to memory-intensive operations in dense prediction tasks, and insufficient consideration of the fixed teacher model with gradient conflicts. 
To handle these limitations, we propose a sample-wise contrastive learning framework that simultaneously preserves intra-sample feature similarity and cross-sample semantics with the optimized contrastive triplets, effectively addressing the memory constraints.

\subsection{Contrastive Learning} 
Contrastive learning stands out as a powerful and popular branch of self-supervised learning, and plenty of efforts have been devoted to contrastive learning \cite{bachman2019learning,henaff2020data}. We merely review relevant works that are concerned with noise-contrastive estimation (NCE) and InfoNCE. Gutmann \textit{et al.} \cite{gutmann2012noise} firstly proposed the theory of noise-contrastive estimation, in which the model learns to discriminate between samples from the actual data distribution and those from a noise distribution, revealing distinct features of the data. In noise-contrastive estimation, increasing the volume of negative samples can largely reduce reliance on the quality of the noise distribution. 
Sohn \textit{et al.} \cite{sohn2016improved} extended this framework with the multi-class N-pair loss, enabling efficient reuse of negative embeddings across iterations.

Subsequent works optimized negative sampling and scalability. For example, 
InstDisc \cite{wu2018unsupervised} introduced a memory bank to store historical instance features, enabling stable negative sampling without recomputation. CPC \cite{oord2018representation} popularized the InfoNCE loss by formalizing contrastive learning as maximizing mutual information between positive pairs. 
Likewise, CMC \cite{tian2020contrastive} leveraged a contrastive loss function to optimize the mutual information across different view representations of the same sample. 
MoCo \cite{he2020momentum} decoupled batch size from negative sample size via a momentum-updated queue and conducted large-scale negative sampling with limited memory. SimCLR \cite{chen2020simple} demonstrated that strong data augmentation paired with a nonlinear projection head achieves state-of-the-art results, albeit with high computational demands.
CLIP \cite{radford2021learning} proposed a large-scale multimodal model pre-trained through the joint training of a text encoder and an image encoder with the InfoNCE loss. While these methods highlight the efficacy of contrastive learning, they rely heavily on large batches of negative samples to ensure the effective separation of representations. 
This dependency creates computational bottlenecks, particularly in dense prediction tasks.
In this paper, we adapt contrastive learning to knowledge distillation by rethinking negative sampling efficiency. Instead of relying on massive negative batches, we propose sample-wise triples that leverage intra-class and inter-class relationships. Our approach maintains the structural benefits of contrastive learning while reducing memory overhead, making it scalable to complex tasks like segmentation and detection.

\section{Approach}
In this section, we first describe some notations. Then, we delve into contrastive knowledge distillation. Lastly, we provide a discussion about the proposed method.

\subsection{Notations}
Knowledge distillation is a technique used to transfer knowledge from a trained teacher model $\mathcal{T}$ to a more compact student model $\mathcal{S}$. 
Here are some notations, given a batch of input samples $\mathcal{X}=[\boldsymbol{x}_0, \boldsymbol{x}_1, \ldots, \boldsymbol{x}_{n-1}]$, we calculate teacher logits $\boldsymbol{T} = \mathcal{T}\left ( \mathcal{X};\theta_{\mathcal{T}} \right )$ and student logits $\boldsymbol{S} = \mathcal{S}\left ( \mathcal{X};\theta_{\mathcal{S}}  \right ) $, where $\boldsymbol{T}=[\boldsymbol{t}_0, \boldsymbol{t}_1, \ldots, \boldsymbol{t}_{n-1}]\in \mathbb{R}^{c\times n}$ and $\boldsymbol{S}=[\boldsymbol{s}_0, \boldsymbol{s}_1, \ldots, \boldsymbol{s}_{n-1}]\in \mathbb{R}^{c\times n}$. 
$\theta_\mathcal{T}$ and $\theta_\mathcal{S}$ are model parameters. $\mathbf{t}_i$ and $\mathbf{s}_i$ represent the teacher and student the logits for the $i$-th sample in which $i \in \mathcal{D}$, $\mathcal{D}=[1,2,\ldots,n-1]$. $n$ denotes the sample count, and $c$ is the number of classes that corresponds to the logit dimension. 
The student model $\mathcal{S}$ learns to mimic the teacher’s output. A typical distillation loss combines task loss and distillation loss.
Formally, the objective of knowledge distillation approaches can be written as,
\begin{equation}
   \mathcal{L}  = \mathcal{L}_{\text{Task}} + \alpha \mathcal{L}_{\text{KD}}.
   \label{eq1}
\end{equation}
$\mathcal{L}_{\text{Task}}$ is the standard cross-entropy loss for image classification, while $\alpha$ balances the classification loss $\mathcal{L}_{\text{Task}}$ and the distillation loss $\mathcal{L}_{\text{KD}}$. In what follows, we will detail the knowledge distillation loss $\mathcal{L}_{\text{KD}}$.

\subsection{Contrastive Knowledge Distillation}
\label{subsec:ckd}
Traditional knowledge distillation methods emphasize class-level contrastive learning by leveraging intra-class and inter-class relationships. 
In contrast, our approach adopts sample-level distillation. To be specific, for each sample, we enforce alignment between 
the student logit $\mathbf{s}_i$ with its corresponding teacher logit $\mathbf{t}_i$, while simultaneously distancing $\mathbf{s}_i$ from negative student logits $\mathbf{s}_j \; (j \neq i)$. This dual mechanism allows our method to distill not only numerical similarities in individual samples but also semantic structural relationships across different samples.
As illustrated in \Cref{fig2}, our contrastive knowledge distillation framework integrates two complementary components. The intra-sample distillation directly aligns paired student-teacher logits for each sample, and the inter-sample distillation employs contrastive separation between the target logit $\mathbf{s}_i$ and irrelevant logits $\mathbf{s}_j$.
In the following sections, we elaborate on the formulation of these constraints.

\subsubsection{Intra-Sample Distillation}
Unlike class-wise methods that minimize intra-class distances for knowledge transfer, our method explicitly leverages
sample-wise information to enhance the similarity between teacher and student logits for each individual sample. 
Here, we introduce the intra-sample loss, denoted as $\mathcal{L}_{\text{intra}}$, to optimize the parameters of the student model $\mathcal{S}$ by minimizing the distance $d(\cdot,\cdot)$ between aligned teacher-student logit pairs. Formally, the loss is defined as,
\begin{equation}
\begin{aligned}
\mathcal{L}_{\text{intra}} 
&=  \frac{1}{n}\sum_{i=0}^{n} d\left ( \boldsymbol{t}_i,\boldsymbol{s}_i \right) \\
&=  \underset{\boldsymbol{x}_i \sim \mathcal{X}}{\mathbb{E}}{ d\left ( \boldsymbol{t}_i,\boldsymbol{s}_i \right)}.
\label{eq2}
\end{aligned}
\end{equation}
For each input sample $\boldsymbol{x}_i \in \mathcal{X}$, the teacher and student models generate logits $\boldsymbol{t}_i=\mathcal{T}\left(\boldsymbol{x}_i;\theta_{\mathcal{T}}\right)$ and $\boldsymbol{s}_i=\mathcal{S}\left(\boldsymbol{x}_i;\theta_{\mathcal{S}}\right)$, respectively. $\boldsymbol{x}_i \sim \mathcal{X}$ means that a data point $\boldsymbol{x}_i$ sampled from the data space $\mathcal{X}$. 
In \Cref{eq2}, the intra-sample loss $\mathcal{L}_{\text{intra}}$ enforces alignment between the student logit $\boldsymbol{s}_i$ and its analytical solution $\boldsymbol{t}_i$. 
While such alignment improves local similarities, earlier research \cite{romero2014fitnets, park2019relational} reveals that exclusive reliance on intra-sample similarity risks overfitting as excessive minimization of the teacher-student gap may degrade generalization to unseen data.  
A phenomenon particularly evident in scenarios dominated by sample-specific optimization.

\begin{figure}[t]
    \centering
    \subfigure[$d(\mathbf{t}_i,\mathbf{s}_i)<d(\mathbf{t}_j,\mathbf{s}_i)$]{    \includegraphics[width=0.4\columnwidth]{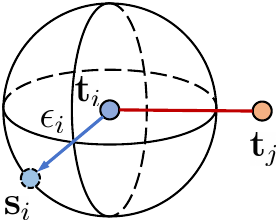}\label{fig3_a} 
    }
    \subfigure[$d(\mathbf{t}_i,\mathbf{s}_i)>d(\mathbf{t}_j,\mathbf{s}_i)$]{  \includegraphics[width=0.4\columnwidth]{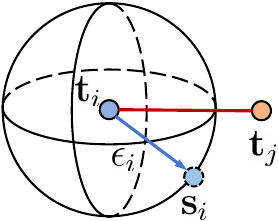}\label{fig3_b}
    }
    \caption{Visualization of the intra-sample alignment. For the sake of simplicity, we assume that the dimension of logits is 3. The spherical surface represents $\mathbf{s}_i$ satisfying the condition $ \epsilon_i=\mathbf{t}_i-\mathbf{s}_i $ for the fixed $ \epsilon_i$.}
    \label{fig3}
\end{figure}

\begin{figure*}[t] 
\centering
    \hfill
\subfigure[$\left(\boldsymbol{s}_i,\boldsymbol{t}_i,\boldsymbol{s}_j\right)$]{
        \includegraphics[width=0.6\columnwidth]{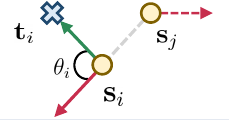}
        \label{fig4_a} 
      }
\subfigure[$\left(\boldsymbol{s}_i,\boldsymbol{t}_i,\boldsymbol{t}_j\right)$]{
        \includegraphics[width=0.6\columnwidth]{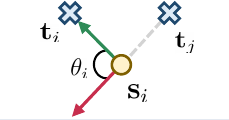}
        \label{fig4_b} 
      }
\subfigure[$\left(\boldsymbol{s}_i,\boldsymbol{t}_i,\boldsymbol{t}_j\right)$]{
        \includegraphics[width=0.6\columnwidth]{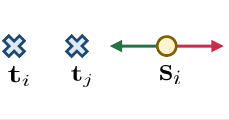}
        \label{fig4_c} 
      }
    \caption{Visualization of two categories of triples. (a) depicts $\left(\boldsymbol{s}_i,\boldsymbol{t}_i,\boldsymbol{s}_j\right)$, while (b) and (c) show $\left(\boldsymbol{s}_i,\boldsymbol{t}_i,\boldsymbol{t}_j\right)$. The blue line indicates that $\boldsymbol{s}_i$ approaches $\boldsymbol{t}_i$, and the red line represents that $\boldsymbol{s}_i$ is far from $\boldsymbol{s}_j$ or $\boldsymbol{t}_j$.}
    \label{fig4} 
\end{figure*}

\begin{figure}[t] 
\centering 
\includegraphics[width=1.0\columnwidth]{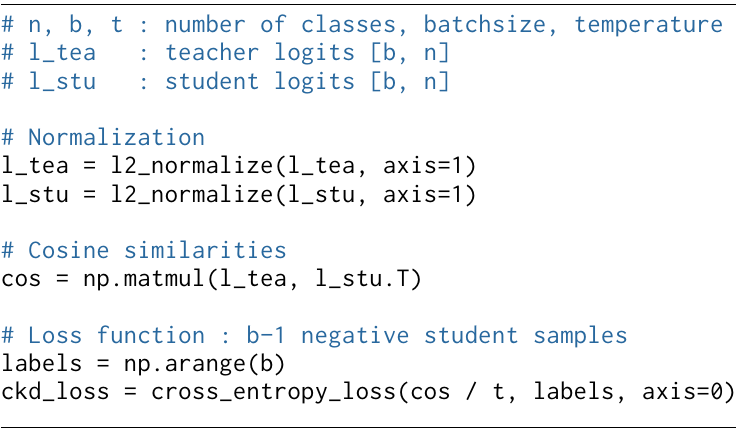} 
\vspace{-0.2cm}
\caption{Pseudo code of CKD in a Numpy-like style.} 
\label{fig5} 
\end{figure}

\Cref{fig3} visualizes the limitations of the intra-sample constraint by analyzing teacher logits $\boldsymbol{t}_i$ and student logits $\boldsymbol{t}_j$, where $\boldsymbol{t}_i\ne\boldsymbol{t}_j$. 
We define the alignment error for the $i$-th sample as $\epsilon_i = \boldsymbol{t}_i - \boldsymbol{s}_i$. 
While $\mathcal{L}_{\text{intra}}$ minimizes $d(\boldsymbol{t}_i,\boldsymbol{s}_i)$, it does not guarantee the inequality $\forall i, d(\boldsymbol{t}_i,\boldsymbol{s}_i)>d(\boldsymbol{t}_j,\boldsymbol{s}_i)$. 
Consequently, there may exist cases where $\epsilon_i$, such that ${\exists} i, d(\boldsymbol{t}_i,\boldsymbol{s}_i)>d(\boldsymbol{t}_j,\boldsymbol{s}_i)$. 
This inequality implies that the student logit $\boldsymbol{s}_i$ becomes closer to an incorrect teacher logit $\boldsymbol{t}_j$ than to its assigned target $\boldsymbol{t}_i$. 

Formally, we consider a family of hyperplanes $\mathcal{H}^{(k)}, k\in \mathbb{N}$ in the student logit space, where each $\boldsymbol{s}_i^{(k)}$ under $\mathcal{H}^{(k)}$ satisfies $\boldsymbol{s}_i^{(k)} \in \mathcal{H}^{(k)}, i\in \mathcal{D} $.
These hyperplanes share the same aggregated loss value across all samples as, 
\begin{equation}
\forall k,h \in  \mathbb{N}\;\;\sum_{i\in \mathcal{D}} d\left ( \boldsymbol{t}_i,\boldsymbol{s}_i^{(k)} \right) =  \sum_{i\in \mathcal{D}} d\left ( \boldsymbol{t}_i,\boldsymbol{s}_i^{(h)} \right).
\label{eq3}
\end{equation}
Each hyperplane $\mathcal{H}^{(k)}$ is parameterized as,
\begin{equation}
\mathcal{H}^{(k)}: \boldsymbol{w}^{(k)} \boldsymbol{s}^{(k)}_i + b^{(k)}  = 0, \; i \in \mathcal{D}, 
\label{eq4}
\end{equation}
where $\boldsymbol{w}^{(k)} \in \mathbb{R}^{c}$ is the normal vector and $b$ is the bias term.

Consider a binary classification scenario with a decision boundary defined by $y = x$. Suppose a teacher logs $\mathbf{t}_i=\left[0.5,0.5\right]$ corresponds to a sample near this boundary. Two student logits $\left[0.4,0.4\right]$ and $\left[0.6,0.6\right]$ yield identical intra-sample losses under mean squared error.
However, their geometric positions relative to the decision boundary differ critically. 
$\left[0.4,0.4\right]$ falls below the boundary, implying misclassification.
$\left[0.6,0.6\right]$ lies above the boundary, preserving correct classification.
These loss values expose a key limitation of intra-sample distillation that $\mathcal{L}_{intra}$ may preserve numerical similarities while degrading task-specific correctness. To mitigate such overfitting risks, we introduce supplementary constraints that enforce semantic consistency beyond intra-sample alignment.

\subsubsection{Inter-Sample Distillation}
Knowledge distillation should guarantee both numerical similarity and semantic consistency between teacher and student logits \cite{tian2020representation,bai2020feature}. However,  conventional intra-sample distillation focuses on the former by directly aligning paired logits. To address this limitation, we introduce inter-sample semantic distillation, which explicitly enforces structural relationships across samples.
The inter-sample constraint $\mathcal{L}_{\text{inter}}$ operates complementarily to $\mathcal{L}_{\text{intra}}$ through contrastive separation by pushing student logits of distinct samples apart in the embedding space.
Formally, $\mathcal{L}_{\text{inter}}$ is defined as,
\begin{equation}
\begin{aligned}
\mathcal{L}_{\text{inter}} 
&=  -\frac{1}{n(n-1)}\sum_{i=0}^{n} \sum_{j\ne i}^{n}  \left[ d\left ( \boldsymbol{s}_i,\boldsymbol{s}_j \right) + d\left ( \boldsymbol{s}_i,\boldsymbol{t}_j \right) \right]\\ 
& = 
-\underset{\begin{subarray}{c}
  {\boldsymbol{x}_i \sim \mathcal{X}},
  {\boldsymbol{x}_j \sim \mathcal{X}}
\end{subarray}}
{\mathbb{E}}{ d\left ( \boldsymbol{s}_i,\boldsymbol{s}_j \right)}-
\underset{\begin{subarray}{c}
  {\boldsymbol{x}_i \sim \mathcal{X}},
  {\boldsymbol{x}_j \sim \mathcal{X}}
\end{subarray}}
{\mathbb{E}}{ d\left ( \boldsymbol{s}_i,\boldsymbol{t}_j \right)}.
\label{eq5}
\end{aligned}
\end{equation}
The objective loss $\mathcal{L}_{\text{inter}}$ enforces dual separation that it pushes the student logit  $\boldsymbol{s}_i$ away from other student logits $\boldsymbol{s}_j$ and their corresponding teacher logits $\boldsymbol{t}_j$.
As shown in \Cref{fig3_b}, this constraint resolves cross-sample misalignment by preventing $\boldsymbol{s}_i$ from collapsing towards irrelevant teacher logits $\boldsymbol{t}_j$ and inter-sample redundancy by ensuring discriminative separation between distinct student logits $\boldsymbol{s}_i$ and $\boldsymbol{s}_j$.
\Cref{eq5} achieves this through simultaneous distance maximization between $\left ( \boldsymbol{s}_i,\boldsymbol{s}_j \right)$ and $\left ( \boldsymbol{s}_i,\boldsymbol{t}_j \right)$.

Generally, for each student logit $\boldsymbol{s}_i$, we construct inter-sample pairs include student-student pairs $\left(\boldsymbol{s}_i, \boldsymbol{s}_j\right)$ or student-teacher pairs $\left( \boldsymbol{s}_i, \boldsymbol{t}_j\right)$. 
We combine the intra-sample pair $\left(\boldsymbol{s}_i, \boldsymbol{t}_i\right)$ and obtain two types of triples $\left(\boldsymbol{s}_i, \boldsymbol{t}_i, \boldsymbol{s}_j\right)$ and $\left(\boldsymbol{s}_i, \boldsymbol{t}_i, \boldsymbol{t}_j\right)$. 
\Cref{fig4} visualizes these two triple categories. 
While previous triple-based methods \cite{tian2020representation, bai2020feature} leverage both triple types by updating all elements \cite{sohn2016improved,he2020momentum,chen2020simple}, our framework differs as the teacher model remains fixed with its outputs $\boldsymbol{t}_i$ unchanged during distillation. 
This constraint introduces a critical challenge that the direct adoption of traditional triplet losses designed for fully trainable models may lead to suboptimal alignment when applied to static teacher outputs. 

Our method exclusively adopts $\left(\boldsymbol{s}_i, \boldsymbol{t}_i, \boldsymbol{s}_j\right)$ for distillation. \Cref{fig4} briefly illustrates the underlying reasoning behind this selection. 
In \Cref{fig4_a}, the student logit $\boldsymbol{s}_i$ approaches $\boldsymbol{t}_i$ for intra-sample alignment while distancing from $\boldsymbol{s}_j$ for inter-sample separation. 
While gradient conflicts may emerge when the angle $\theta_i > \pi/2$, these can be mitigated by adjusting the trainable $\boldsymbol{s}_j$ during optimization. 
However, in \Cref{fig4_b}, the conflict remains unresolved due to the immutable nature of $\boldsymbol{t}_j$. 
In extreme cases\Cref{fig4_c}, opposing gradients from $\boldsymbol{t}_i$ and $\boldsymbol{t}_j$ destabilize training, as neither term can be adaptively regularized.
From \Cref{fig4}, we conclude that inter-sample interactions must involve trainable student logits to ensure stable optimization and that the triple $\left(\boldsymbol{s}_i, \boldsymbol{t}_i, \boldsymbol{t}_j\right)$ may induce gradient conflicts. Consequently, our method exclusively adopts  $\left(\boldsymbol{s}_i, \boldsymbol{t}_i, \boldsymbol{s}_j\right)$, and therefore formulate the inter-sample objective as the following equation,
\begin{equation}
\begin{aligned}
\mathcal{L}_{\text{inter}} 
&=  -\frac{1}{n(n-1)}\sum_{i=0}^{n} \sum_{j\ne i}^{n}  d\left ( \boldsymbol{s}_i,\boldsymbol{s}_j \right) \\ 
& = 
-\underset{\begin{subarray}{c}
  {\boldsymbol{x}_i \sim \mathcal{X}},
  {\boldsymbol{x}_j \sim \mathcal{X}}
\end{subarray}}
{\mathbb{E}}{ d\left ( \boldsymbol{s}_i,\boldsymbol{s}_j \right)}.
\label{eq6}
\end{aligned}
\end{equation}

Finally, we unify the intra-sample alignment constraint \Cref{eq2} and inter-sample separation constraint \Cref{eq6} through a weighted summation by using a balancing parameter $\beta$ as,
\begin{align}
\mathcal{L}_{\text{KD}} = \mathcal{L}_{\text{intra}} + \beta \mathcal{L}_{\text{inter}}.
\label{eq7}
\end{align}
The objective \Cref{eq7} achieves intra-sample maximization by enhancing the similarity between teacher and student logits for the same sample and inter-sample minimization by suppressing correlations between the logits of distinct samples. This is realized through contrastive learning with positive pairs pulled closer via $\mathcal{L}_{intra}$ and negative pairs pushed apart via $\mathcal{L}_{inter}$.

\subsubsection{Contrastive Formulation}
Notably, the proposed structure-preserving distillation framework exhibits an objective akin to contrastive learning \cite{oord2018representation,xu2020knowledge}. 
This connection motivates us to reformulate our objective into a sample-wise contrastive learning framework. Formally, $\mathcal{L}_{\text{KD}}$ can be expressed as,
\begin{align}
\begin{split}
\mathcal{L}_{\text{KD}}  &= 
\underset{\boldsymbol{x}_i \sim \mathcal{X}}{\mathbb{E}}
{   d\left ( \boldsymbol{t}_i,\boldsymbol{s}_i \right)} - 
\beta \underset{\begin{subarray}{c}
  {\boldsymbol{x}_i \sim \mathcal{X}} \\
  {\boldsymbol{x}_j \sim \mathcal{X}}
  \end{subarray}}{\mathbb{E}}{ d\left ( \boldsymbol{s}_i,\boldsymbol{s}_j \right) } .
\end{split}
\label{eq8}
\end{align} 
For simplicity, we employ the distance metric $f(\boldsymbol{u},\boldsymbol{v})$ to quantify the similarity between $\boldsymbol{u}$ and $\boldsymbol{v}$, defining $f(\boldsymbol{u},\boldsymbol{v}) = - d(\boldsymbol{u},\boldsymbol{v})$. The loss function can be derived as,
\begin{align}
\mathcal{L}_{\text{KD}}&=\underset{\boldsymbol{x}_i \sim \mathcal{X}}{\mathbb{E}}\left [-(f\left ( \boldsymbol{t}_i,\boldsymbol{s}_i \right)-\beta\underset{\boldsymbol{x}_j \sim \mathcal{X}}{\mathbb{E}}f\left ( \boldsymbol{s}_i,\boldsymbol{s}_j \right))\right ].
\label{eq9}
\end{align}
where higher values indicate greater similarity.  Employing \Cref{eq9} directly for optimization may result in gradient conflicts or diminish training efficiency. To rewrite it in the format of contrastive loss, we incorporate the natural logarithm for derivation as,

\begin{align}
\mathcal{L}_{\text{KD}}&=\underset{\boldsymbol{x}_i \sim \mathcal{X}}{\mathbb{E}}
\ln_{}\left [{
\frac{
\underset{\boldsymbol{x}_j \sim \mathcal{X}}{\mathbb{E}}e^{\beta f\left ( \boldsymbol{s}_i,\boldsymbol{s}_j \right)}
}{e^{ f\left ( \boldsymbol{t}_i,\boldsymbol{s}_i \right)}}
}  + 
\frac{e^{ f\left ( \boldsymbol{t}_i,\boldsymbol{s}_i \right)}}{
e^{ f\left ( \boldsymbol{t}_i,\boldsymbol{s}_i \right)}
}
-1
\right ]
\label{eq10}
\\
&\simeq \underset{\boldsymbol{x}_i \sim \mathcal{X}}{\mathbb{E}}\left [   -\ln_{}{\frac{e^{ f\left ( \boldsymbol{t}_i,\boldsymbol{s}_i \right)}}{e^{ f\left ( \boldsymbol{t}_i,\boldsymbol{s}_i \right)}+\underset{\boldsymbol{x}_j \sim \mathcal{X}}{\mathbb{E}}e^{\beta f\left ( \boldsymbol{s}_i,\boldsymbol{s}_j \right)}}}\right ]
\label{eq11}
\\
&\simeq \underset{\boldsymbol{x}_i \sim \mathcal{X}}{\mathbb{E}}\left [   -\ln_{}{\frac{e^{ f\left ( \boldsymbol{t}_i,\boldsymbol{s}_i \right)/\tau}}{e^{ f\left ( \boldsymbol{t}_i,\boldsymbol{s}_i \right)/\tau }+\underset{\boldsymbol{x}_j \sim \mathcal{X}}{\mathbb{E}}e^{\beta f\left ( \boldsymbol{s}_i,\boldsymbol{s}_j \right)/\tau}}}\right ].
\label{eq12}
\end{align}
Here, the teacher logit $\boldsymbol{t}_i$ and its corresponding student logit $\boldsymbol{s}_i$ form a positive pair, while $\boldsymbol{s}_i$ and $\boldsymbol{s}_j$ form a negative pair. Next, we will explore the benefits of this contrastive formulation.

\begin{figure*}[t] 
\centering
      \subfigure[Traditional class-wise contrast]{
            \includegraphics[width=0.6\columnwidth]{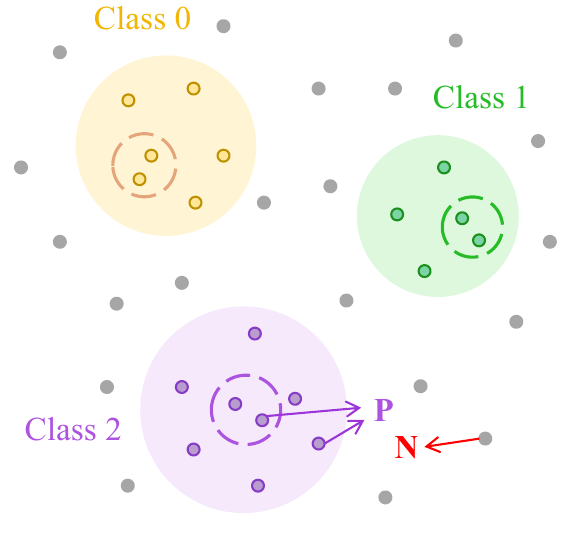}
            \label{fig6_a} 
          }
     \subfigure[Typical class-wise contrast: CRD]{
            \includegraphics[width=0.6\columnwidth]{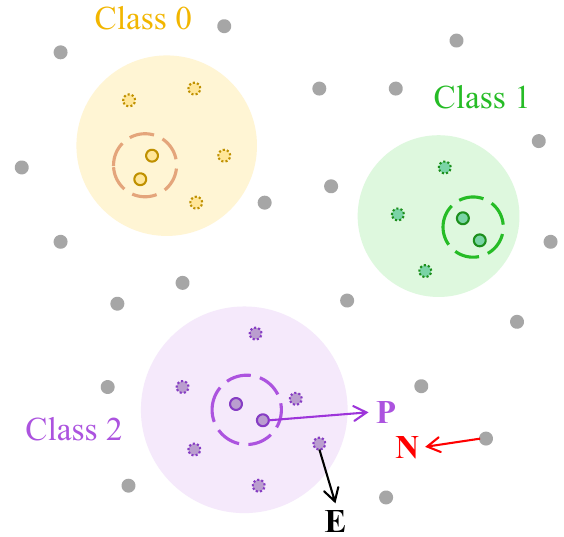}
            \label{fig6_b} 
        }
      \subfigure[Proposed sample-wise contrast]{
            \includegraphics[width=0.6\columnwidth]{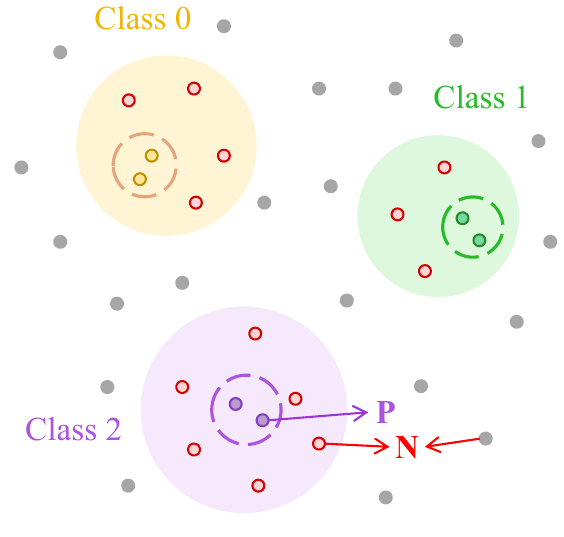}
            \label{fig6_c} 
          }
    \caption{Visualization of two categories of contrastive learning: Class-wise vs. Sample-wise. CRD is a typical class-wise formulation. Different background colors represent various sample categories, with P denoting positive samples and N denoting negative samples. The symbol E indicates that samples from the same class are excluded during training. The dashed circle contains teacher and student representations of the same sample}
    \label{fig6} 
\end{figure*}

As shown in \Cref{eq9}, to compute the expected value $f(\boldsymbol{s}_i, \boldsymbol{s}_j)$ for each $\boldsymbol{x}_i$, we sample negative pairs from the data space $\mathcal{X}$. 
Following vanilla KD \cite{hinton2015distilling}, we introduce a temperature parameter $\tau$ to soften the logit distribution defined in \Cref{eq12}. 
The parameter $\beta$ dynamically adjusts the trade-off between intra-sample alignment and inter-sample separation.
We assign larger $\beta$ values when semantic separation through negative pairs is critical and smaller values otherwise.
Note that this balance can also be achieved through adaptive negative sampling. 
For example, in scenarios with dominant negative spaces, we increase the number of sampled negative pairs to improve space estimation and weight adjustment. When negative pairs are less informative, we reduce their sampling frequency.
Contrastive learning enables efficient parallel processing of batch data.
In our formulation, we construct training batches of data samples, each containing one positive sample and $B-1$ negative samples, where $B$ denotes the batch size. 
As shown in \Cref{fig2}, the number of negative samples depends on $B$. 
Therefore, adjusting $B$ thus provides explicit control over the intra-inter sample balance.

Additionally, we replace $f\left ( \boldsymbol{t}_i,\boldsymbol{s}_j \right)$ with $f\left ( \boldsymbol{s}_i,\boldsymbol{s}_j \right)$. 
This substitution is justified by two key considerations. 
First, our method primarily aligns $\boldsymbol{s}_i$ and $\boldsymbol{t}_i$, such that $\boldsymbol{t}_i= \boldsymbol{s}_i$. We have $\boldsymbol{t}_i= \boldsymbol{s}_i\Rightarrow f\left ( \boldsymbol{t}_i,\boldsymbol{s}_j \right)=f\left ( \boldsymbol{s}_i,\boldsymbol{s}_j \right)$. 
Second, $f\left ( \boldsymbol{s}_i,\boldsymbol{s}_j \right)$ involves two trainable variables, whereas  $f\left ( \boldsymbol{t}_i,\boldsymbol{s}_j \right)$ depends only on $ \boldsymbol{s}_j$. This optimization simplifies gradient propagation through $ \boldsymbol{s}_j$. 
Specifically, we treat $\boldsymbol{t}_i$ as an anchor logit, and its positive logit $ \boldsymbol{s}_i$ and negative logit $ \boldsymbol{s}_j$ serve as learnable positive and negative logits, respectively. 
This design naturally connects with the InfoNCE loss \cite{oord2018representation}, enabling batch-efficient training through $B \times B $ (teacher, student) pairings per batch analogous to CLIP \cite{radford2021learning}.
Here, batch size $B$ controls the number of negative samples, balancing alignment accuracy and computational cost.

Consequently, we derive the final loss function as,
\begin{align}
\mathcal{L}_{\text{KD}}
&\simeq \underset{\boldsymbol{x}_i \sim \mathcal{X}}{\mathbb{E}}\left [-\ln{\frac{e^{ f\left ( \boldsymbol{t}_i,\boldsymbol{s}_i \right)/\tau}}{e^{ f\left ( \boldsymbol{t}_i,\boldsymbol{s}_i \right)/\tau }+ {\textstyle \sum\limits ^{B}}  e^{ f\left ( \boldsymbol{s}_i,\boldsymbol{s}_j \right)/\tau}}}\right ]
\label{eq13}
\\
&\simeq \underset{\boldsymbol{x}_i \sim \mathcal{X}}{\mathbb{E}}\left [-\ln{\frac{e^{ f\left ( \boldsymbol{t}_i,\boldsymbol{s}_i \right)/\tau}}{e^{ f\left ( \boldsymbol{t}_i,\boldsymbol{s}_i \right)/\tau }+ {\textstyle \sum\limits ^{B}}  e^{ f\left ( \boldsymbol{t}_i,\boldsymbol{s}_j \right)/\tau}}}\right ].
\label{eq14}
\end{align}

In \Cref{fig5}, we present the pseudo-code for the proposed CKD in a Numpy-like style.

\begin{theorem}
\label{throrm1}
Given triples $\left(\boldsymbol{t}_i,\boldsymbol{s}_i,\boldsymbol{s}_j\right)$, the loss function $\mathcal{L}_i$ is defined as,
\begin{align}
\mathcal{L}_i=-\ln{\frac{e^{ f\left ( \boldsymbol{t}_i,\boldsymbol{s}_i \right)/\tau}}{e^{ f\left ( \boldsymbol{t}_i,\boldsymbol{s}_i \right)/\tau }+ {\textstyle \sum\limits ^{B}}  e^{ f\left ( \boldsymbol{t}_i,\boldsymbol{s}_j \right)/\tau}}}.
\label{eq15}
\end{align}
The gradients $\nabla_{\boldsymbol{s}_i}\mathcal{L}_i$ and $\nabla_{\boldsymbol{s}_j}\mathcal{L}_i$ are proportional to $g_i$ that is formed as,
\begin{align}
g_i = 1-\frac{e^{ f\left ( \boldsymbol{t}_i,\boldsymbol{s}_i \right)/\tau}}{e^{ f\left ( \boldsymbol{t}_i,\boldsymbol{s}_i \right)/\tau }+ {\textstyle \sum\limits ^{B}}  e^{ f\left ( \boldsymbol{t}_i,\boldsymbol{s}_j \right)/\tau}}.
\label{eq16}
\end{align}
\end{theorem}
\begin{proof}
\renewcommand{\qedsymbol}{}
The detailed proof can be found in \cite{yeh2022decoupled}.
\end{proof}

\begin{theorem}
\label{throrm2}
For class-wise contrastive learning, $\boldsymbol{x}_i$ and $\boldsymbol{x}_j$ belong to different classes and $f\left( \boldsymbol{t}_i, \boldsymbol{s}_i \right) - f\left( \boldsymbol{t}_i, \boldsymbol{s}_j \right)=\rho_j$. When decreasing $B$, the gradients $\nabla_{\boldsymbol{s}_i}\mathcal{L}_i$ also decreases.
\end{theorem}
\begin{proof}
\renewcommand{\qedsymbol}{}
From \Cref{eq16}, we can obtain following equations,
\begin{align}
g_i 
&= 1-\frac{e^{ f\left ( \boldsymbol{t}_i,\boldsymbol{s}_i \right)/\tau}}{e^{ f\left ( \boldsymbol{t}_i,\boldsymbol{s}_i \right)/\tau }+ {\textstyle \sum\limits ^{B}}  e^{ f\left ( \boldsymbol{t}_i,\boldsymbol{s}_j \right)/\tau}}
\label{eq17}
\\
&= 1-\frac{1}{1+ {\textstyle \sum\limits ^{B}}  e^{ -\rho_j/ \tau} }
\label{eq18}
\\
&\simeq 1-\frac{1}{1+ B \times e^{ -\rho_j/ \tau} }.
\label{eq19}
\end{align}
When the batch size $B$ decreases, fewer negative samples are sampled, leading to an increase in $\rho_j$ and a decrease in $g_i$. This, in turn, results in a reduction of gradients $\nabla_{\boldsymbol{s}_i}\mathcal{L}_i$.
\end{proof}

As proven in \Cref{throrm2}, when positive and negative samples exhibit weak correlations, the similarity measure 
$f(\boldsymbol{t}_i,\boldsymbol{s}_j)$ approaches its minimum. This suppresses the backpropagation gradient magnitude
$\nabla_{\boldsymbol{s}_i}\mathcal{L}_i$, leading to inefficient model parameter updates. 
This limitation forces class-wise contrastive methods \cite{sohn2016improved,he2020momentum} to rely on large batch sizes for gradient stabilization by amplifying signal-to-noise ratios through abundant negative samples and for correlation enhancement by mitigating weak correlation scenarios via batch sampling.
Such strategies, while effective, incur substantial computational costs in contrastive learning \cite{yeh2022decoupled}.

\begin{algorithm}[t]
\caption{Training procedure of the proposed CKD}
\small
\label{alg1}
\DontPrintSemicolon
\SetAlgoLined
\KwIn{Images $\mathcal{X}={\{\boldsymbol{x}_i\}_{i=1}^{n-1}}$, annotations $\mathcal{Y}={\{y_i\}_{i=1}^{n-1}}$, and the teacher model $\mathcal{T}(\mathcal{X};\theta_{\mathcal{T}})$  
}
\KwOut{The student model $\mathcal{S}(\mathcal{X};\theta_{\mathcal{S}})$}
 \For{epoch $\in \{1,2,\ldots,\text{Epoch}\}$}{
 \While{Randomly sample a batch $\boldsymbol{B} \in \mathcal{X}$ and $\boldsymbol{y} \in \mathcal{Y}$}
  {
    \emph{\% Compute teacher and student logits}\;
    $\boldsymbol{t}_{B} = \mathcal{T}(\boldsymbol{B};\theta_{\mathcal{T}})$\;
    $\boldsymbol{s}_{B}= \mathcal{S}(\boldsymbol{B};\theta_{\mathcal{S}})$\;
   
    \emph{\% Compute knowledge distillation loss}\;
    $\boldsymbol{t^{\prime}}_{B} = $ Normalize($\boldsymbol{t}_{B}$)\;
    $\boldsymbol{s^{\prime}}_{B} = $ Normalize($\boldsymbol{s}_{B}$)\;
    $\mathcal{L}_{\text{KD}}\gets$  SampleWiseContrastive($\boldsymbol{s^{\prime}}_{B}, \boldsymbol{t^{\prime}}_{B}$) \; 
    
    \emph{\% Compute classification loss}\;
    $\boldsymbol{\bar{y}} \gets $ Softmax($\boldsymbol{s}_B$)\;
    $\mathcal{L}_{\text{Task}} \gets $ CrossEntropy($\boldsymbol{\bar{y}}$, $\boldsymbol{y}$)\;
    
    \emph{\% Apply the loss in \Cref{eq7}}\;
    $\mathcal{L} \gets  \mathcal{L}_{\text{Task}} + \alpha\mathcal{L}_{\text{KD}}$ \;
    $\theta_{\mathcal{S}} \gets \theta_{\mathcal{S}} - \eta \  \partial \mathcal{L} / \partial \theta_{\mathcal{S}}$ \;
  }
}
\end{algorithm}

\begin{theorem}
\label{throrm3}
For sample-wise contrastive learning, $\exists j$ where $\boldsymbol{x}_i$ and $\boldsymbol{x}_j$ share the same class label, such that $\exists j, \;  f\left( \boldsymbol{t}_i, \boldsymbol{s}_i \right) \sim f\left( \boldsymbol{t}_i, \boldsymbol{s}_j \right)$, large batch sizes are not crucial. 
\end{theorem}

\begin{proof}
\renewcommand{\qedsymbol}{}
For sample-wise triples where $\exists j, f\left( \boldsymbol{t}_i, \boldsymbol{s}_i \right) \sim f\left( \boldsymbol{t}_i, \boldsymbol{s}_j \right)$, we can obtain $ \rho_j\sim 0$ and a larger $B \times e^{ -\rho_j}$ is achieved, when sampling negative samples from the same class. Our method eliminates the need for comprehensive negative space sampling with large batch sizes. This finding aligns with the theoretical insights established by Robinson et al.\cite{robinson2020contrastive}, who demonstrated that difficult negative samples that are close to the anchor can provide more information and accelerate the gradient update.
\end{proof}

Some works have revealed that each image is distinctive and can vary significantly from other images within the same semantic class \cite{wu2018unsupervised, malisiewicz2011ensemble}. It suggests that sample-wise contrastive learning could potentially outperform class-wise approaches by preserving intrinsic structural information in this case. Specifically, class-wise methods, which enforce uniform aggregation of intra-class samples toward a centroid, risk distorting the natural semantic structures inherent in the data.

Other research suggests that negative samples in the same class would harm contrastive learning, as the distance between same-class samples may become excessively large \cite{khosla2020supervised}. Our method addresses this challenge through teacher-guided structural constraints.
When feature representations $\boldsymbol{s}_i$ and $\boldsymbol{s}_j$ from the same class move away from each other, their respective convergence towards corresponding teacher anchors $\boldsymbol{t}_i$ and $\boldsymbol{t}_j$ effectively constrains their relative positions. Consequently, this process leads to the alignment between the student’s intra-class semantic structure and that of the teacher.

\Cref{fig6} illustrates the differences between class-wise and sample-wise contrastive paradigms. 
Traditional class-wise approaches like CRD \cite{tian2020representation} employ memory banks to construct negative pairs through class exclusion, i.e., $\mathbf{x}_j \in \mathcal{C}_k$ where $k\neq y_i$ and $\mathcal{C}_k$ is $k$-th category of samples with
\begin{equation}
\mathbb{E}\left[f\left(\boldsymbol{t}_i, \boldsymbol{s}_j\right)\right] \rightarrow 0 \quad \forall j \in \mathcal{C}_{k \neq y_i}.
\end{equation}
The class-wise formulation generates negative pairs from different classes, and these negative samples often exhibit low similarity due to diverse semantic content, forcing gradient enhancement through batch size scaling $g_i \propto B $.
In contrast, the sample-wise formulation may sample negative pairs from the same class as,
\begin{equation}
\mathbb{E}\left[f\left(\boldsymbol{t}_i, \boldsymbol{s}_j\right)\right] \gg 0 .
\end{equation}
This non-zero similarity naturally amplifies gradient signals without batch size dependency.
Therefore, our method does not require a large batch size.
In fact, the proposed method exhibits batch sensitivity. 
Excessive $B$ introduces over-regularization while Insufficient $B$ empirically degrades convergence stability.

\begin{table*}[t]
    \centering
    \caption{Classification accuracy (\%) on the CIFAR-100 validation set. "Homogeneous Architectures" and "Heterogeneous Architectures" represent whether teacher and student networks share the same or different architectures, respectively. $\mathrm{\Delta}$ denotes the improved performance compared to vanilla KD. All results are the average over five trials.}
    \footnotesize
    \resizebox{1.0\linewidth}{!}{
    \begin{tabular}{l@{\hspace{2em}}cccccc@{\hspace{3em}}ccccc}
    \toprule
         Method & \multicolumn{6}{c}{Homogeneous Architectures} & \multicolumn{5}{c}{Heterogeneous Architectures} \\ 
    \midrule
         Teacher & RN56 & RN110 & RN32$\times$4 & W-40-2 & W-40-2 & VGG13 & RN32$\times$4 & W-40-2 & VGG13 & RN50 & RN32$\times$4 \\
         Acc & 72.34 & 74.31 & 79.42 & 75.61 & 75.61 & 74.64 & 79.42 & 75.61 & 74.64 & 79.34 & 79.42 \\
    \midrule
    
         Student & RN20 & RN32 & RN8$\times$4 & W-16-2 & W-40-1 & VGG8 & SN-V1 & SN-V1 & MN-V2 & MN-V2 & SN-V2\\
         Acc & 69.06 & 71.14 & 72.50 & 73.26 & 71.98 & 70.36 & 70.50 & 70.50 & 64.60 & 64.60 & 71.82 \\
    \midrule
         
         \rowcolor[rgb]{0.8,0.8,0.8} 
         \multicolumn{12}{l}{\emph{Feature-based Distillation}} 
         \vspace{2pt} \\
         FitNet\cite{romero2014fitnets} & 69.21 & 71
         06 & 73.50 & 73.58 & 72.24 & 71.02 & 73.59 & 73.73 & 64.14 & 63.16 & 73.54 \\
         AT\cite{komodakis2017paying} & 70.55 & 72.31 & 73.44 & 74.08 & 72.77 & 71.43 & 71.73 & 73.32 & 59.40 & 58.58 & 72.73 \\
         RKD\cite{park2019relational} & 69.61 & 71.82 & 71.90 & 73.35 & 72.22 & 71.48 & 72.28 & 72.21 & 64.52 & 64.43 & 73.21 \\
         CRD\cite{tian2020representation} & 71.16 & 73.48 & 75.51 & 75.48 & 74.14 & 73.94 & 75.11 & 76.05 & 69.73 & 69.11 & 75.65 \\
         OFD\cite{heo2019comprehensive} & 70.98 & 73.23 & 74.95 & 75.24 & 74.33 & 73.95 & 75.98 & 75.85 & 69.48 & 69.04 & 76.82 \\
         ReviewKD\cite{chen2021distilling} & 71.89 & 73.89 & 75.63 & 76.12 & 75.09 & 74.84 & 77.45 & 77.14 & \textbf{70.37} & 69.89 & 77.78 \\
    \midrule
          \rowcolor[rgb]{0.8,0.8,0.8} 
          \multicolumn{12}{l}{\emph{Logits-based Distillation}} \vspace{2pt} \\
          KD\cite{hinton2015distilling} & 70.66 & 73.08 & 73.33 & 74.92 & 73.54 & 72.98 & 74.07 & 74.83 & 67.37 & 67.35 & 74.45 \\
          DML\cite{zhang2018deep} & 69.52 & 72.03 & 72.12 & 73.58 & 72.68 & 71.79 & 72.89 & 72.76 & 65.63 & 65.71 & 73.45 \\
          TAKD\cite{mirzadeh2020improved} & 70.83 & 73.37 & 73.81 & 75.12 & 73.78 & 73.23 & 74.53 & 75.34 & 67.91 & 68.02 & 74.82 \\
          DKD\cite{zhao2022decoupled} & 71.97 & 74.11 & 76.32 & 76.24 & 74.81 & 74.68 & 76.45 & 76.70 & 69.71 & 70.35 & 77.07 \\
          \textbf{CKD} & \textbf{72.12} & \textbf{74.48} & \textbf{76.76} & \textbf{76.28} & \textbf{75.14} & \textbf{74.98} & \textbf{77.72} & \textbf{77.47} & 70.11 & \textbf{70.62} & \textbf{78.18} \\
          $\mathrm{\Delta}$ & \textcolor[RGB]{0, 166, 79}{+1.46} & \textcolor[RGB]{0, 166, 79}{+1.40} & \textcolor[RGB]{0, 166, 79}{+3.43} & \textcolor[RGB]{0, 166, 79}{+1.36} & \textcolor[RGB]{0, 166, 79}{+1.60} & \textcolor[RGB]{0, 166, 79}{+2.00} &  \textcolor[RGB]{0, 166, 79}{+3.65} & \textcolor[RGB]{0, 166, 79}{+2.64} & \textcolor[RGB]{0, 166, 79}{+2.74} & \textcolor[RGB]{0, 166, 79}{+3.27} & \textcolor[RGB]{0, 166, 79}{+3.73} \\
    \bottomrule
    \end{tabular}
    }
    \label{tab1}
\end{table*}

\begin{table*}[t]
    \centering
    \caption{Top-1 and Top-5 accuracy (\%) on the ImageNet-1K validation set.(a) ResNet34 serves as the teacher, while ResNet18 acts as the student. (b) ResNet50 serves as the teacher, while MobileNetV1 acts as the student. $\mathrm{\Delta}$ denotes the performance improvement compared to vanilla KD. All results are the average over three trials.}
    \begin{tabular}{cccc|cccc|cccc}
    \toprule
        \multicolumn{4}{c}{Baselines} & \multicolumn{4}{|c|}{Feature-based Distillation} & \multicolumn{4}{c}{Logits-based Distillation} \\
    \midrule
        Setting & Metric & Teacher & Student & AT\cite{Zagoruyko2016PayingMA} & OFD\cite{heo2019comprehensive} & CRD\cite{tian2020representation} & ReviewKD\cite{chen2021distilling} & KD\cite{hinton2015distilling} & DKD\cite{zhao2022decoupled} & \textbf{CKD} & $\mathrm{\Delta}$ \\
     \midrule
         \multirow{2}{*}{(a)} & Top-1 & 73.31 & 69.75 & 70.69 & 70.81 & 71.17 & 71.61 & 70.66 & 71.70 & \textbf{72.24} & \textcolor[RGB]{0, 166, 79}{+1.58} \\
         & Top-5 & 91.42 & 89.07 & 90.01 & 89.98 & 90.13 & 90.51 & 89.88 & 90.41 & \textbf{90.81} & \textcolor[RGB]{0, 166, 79}{+0.93} \\
     \midrule
         \multirow{2}{*}{(b)} & Top-1 & 76.16 & 68.87 & 69.56 & 71.25 & 71.37 & 72.56 & 68.58 & 72.05 & \textbf{72.97} & \textcolor[RGB]{0, 166, 79}{+4.39} \\
         & Top-5 & 92.86 & 88.76 & 89.33 & 90.34 & 90.41 & 91.00 & 88.98 & 91.05 & \textbf{91.36} & \textcolor[RGB]{0, 166, 79}{+2.38} \\
    \bottomrule
    \end{tabular}
    \label{tab2}
\end{table*}

\begin{table*}[t]
    \centering
    \caption{Top-1 and Top-5 accuracy (\%) on the Places365 validation set.(a) ResNet34 serves as the teacher, while ResNet18 acts as the student. (b) ResNet50 serves as the teacher, while MobileNetV1 acts as the student. $\mathrm{\Delta}$ denotes the performance improvement compared to vanilla KD. All results are the average over three trials.}
    \begin{tabular}{cccc|cccc|cccc}
    \toprule
        \multicolumn{4}{c|}{Baselines} & \multicolumn{4}{c|}{Feature-based Distillation} & \multicolumn{4}{c}{Logits-based Distillation} \\
    \midrule
        Setting & Metric & Teacher & Student & AT\cite{Zagoruyko2016PayingMA} & OFD\cite{heo2019comprehensive} & CRD\cite{tian2020representation} & ReviewKD\cite{chen2021distilling} & KD\cite{hinton2015distilling} & DKD\cite{zhao2022decoupled} & \textbf{CKD} & $\mathrm{\Delta}$ \\
    \midrule
         \multirow{2}{*}{(a)} & Top-1 & 54.57 & 54.07 & 54.62 & 55.22 & 54.92 & 55.33 & 55.30 & 55.43 & \textbf{55.47} & \textcolor[RGB]{0, 166, 79}{+0.17} \\
         & Top-5 & 85.05 & 84.44 & 84.90 & 85.25 & 85.15 & 85.42 & \textbf{85.49} & 85.35 & 85.38 & \textcolor[RGB]{0, 166, 79}{-0.11} \\
    \midrule
         \multirow{2}{*}{(b)} & Top-1 & 55.51 & 51.74 & 52.05 & 55.19 & 54.71 & 54.83  & 53.07 & 55.25 & \textbf{55.51} & \textcolor[RGB]{0, 166, 79}{+2.47} \\
         & Top-5 & 85.51 & 82.29 & 82.67 & 85.16 & 84.86 & 85.13  & 83.48 & \textbf{85.50} & 85.41 & \textcolor[RGB]{0, 166, 79}{+1.93} \\

    \bottomrule
    \end{tabular}
    \label{tab3}
\end{table*}

Given the logit matrix $\boldsymbol{T} \in \mathbb{R}^{c\times n}$, and assume that there are $c$ classes, each class containing $m$ samples. The negative logit space $\boldsymbol{T}_{/i} \in \mathbb{R}^{c\times (n-1)}$ for the $i$-th sample satisfies,
\begin{align}
\text{rank}\left(\boldsymbol{T}_{/i} \right) \le \min{\left(c, n-1\right)} = c,  \; \textit{ s.t. }  c \ll n.
\label{eq20}
\end{align}
This rank constraint implies that the effective degrees of freedom in the negative space scale with the class count $c$ rather than the sample count $n$. Consequently, the intrinsic dimensionality of negative pairs is governed by category correlations, reducing the required sampling complexity from $\mathcal{O}(n)$ to $\mathcal{O}(c)$.
Furthermore, compared to class-wise contrastive learning approaches, our method achieves efficient training through well-designed triples. 

\begin{table*}[t]
    \centering
    \caption{Object detection performance on the MS-COCO dataset. We employ the two-stage Faster RCNN with FPN as the detector. Teacher-student pairs are ResNet101 \& ResNet-18, ResNet101 \& ResNet50, and ResNet50 \& MobileNetV2, respectively. $\mathrm{\Delta}$ represents the improved performance compared to ReviewKD.
    }
     \small
    \begin{tabular}{@{\hspace{1em}}l@{\hspace{3em}}ccc@{\hspace{3em}}ccc@{\hspace{3em}}ccc@{\hspace{1em}}}
    \toprule
        \multirow{2}{*}{Methods} & \multicolumn{3}{@{\hspace{1.5em}}l}{RN101 \& RN18} & \multicolumn{3}{@{\hspace{1.5em}}l}{RN101 \& RN50} & \multicolumn{3}{@{\hspace{1.5em}}l}{RN50 \& MN-V2} \\
         & AP & AP$_{50}$ & AP$_{75}$ & AP & AP$_{50}$ & AP$_{75}$ & AP & AP$_{50}$ & AP$_{75}$ \\
    \midrule
         Teacher & 42.04 & 62.48 & 45.88 & 42.04 & 62.48 & 45.88 & 40.22 & 61.02 & 43.81 \\
         Student & 33.26 & 53.61 & 35.26 & 37.93 & 58.84 & 41.05 & 29.47 & 48.87 & 30.90 \\
    \midrule
        KD \cite{hinton2015distilling} & 33.97 & 54.66 & 36.62 & 38.35 & 59.41 & 41.71 & 30.13 & 50.28 & 31.35 \\
        FitNet\cite{romero2014fitnets} & 34.13 & 54.16 & 36.71 & 38.76 & 59.62 & 41.80 & 30.20 & 49.80 & 31.69 \\
        FGFI\cite{wang2019distilling} & 35.44 & 55.51 & 38.17 & 39.44 & 60.27 & 43.04 & 31.16 & 50.68 & 32.92 \\
        ReviewKD\cite{chen2021distilling} & 36.75 & 56.72 & 34.00 & 40.36 & 60.97 & 44.08 & 33.71 & 53.15 & 36.13 \\
        DKD\cite{zhao2022decoupled} & 35.05 & 56.60 & 37.54 & 39.25 & 60.90 & 42.73 & 32.34 & 53.77 & 34.01 \\
        DKD+ReviewKD & 37.01 & 57.53 & 39.85 & 40.65 & 61.51 & 44.44 & 34.35 & 54.89 & 36.61 \\
    \midrule
        \textbf{CKD} & 35.08 & 55.47 & 38.01 & 39.71 & 60.64 & 43.31 & 31.72 & 50.01 & 33.62 \\
        \textbf{CKD+ReviewKD} & \textbf{37.65} & \textbf{57.91} & \textbf{40.71} & \textbf{41.45} & \textbf{61.93} & \textbf{45.31} & \textbf{35.31} & \textbf{55.35} & \textbf{37.88} \\
        $\mathrm{\Delta}$ & \textcolor[RGB]{0, 166, 79}{+0.90} & \textcolor[RGB]{0, 166, 79}{+1.19} & \textcolor[RGB]{0, 166, 79}{+6.71} & \textcolor[RGB]{0, 166, 79}{+1.09} & \textcolor[RGB]{0, 166, 79}{+0.96} & \textcolor[RGB]{0, 166, 79}{+1.23} & \textcolor[RGB]{0, 166, 79}{+1.60} & \textcolor[RGB]{0, 166, 79}{+2.20} & \textcolor[RGB]{0, 166, 79}{+1.75} \\
    \bottomrule
    \end{tabular}
    \label{tab5}
\end{table*}

\begin{table*}[t]
    \centering
    \caption{Instance segmentation performance on the MS-COCO dataset. We take Mask R-CNN with FPN as the detector. Teacher-student pairs are ResNet101 \& ResNet-18, ResNet101 \& ResNet50, and ResNet50 \& MobileNetV2, respectively. $\mathrm{\Delta}$ represents the improved performance compared to ReviewKD.
    }
     \small
    \begin{tabular}{@{\hspace{1em}}l@{\hspace{4em}}c@{\hspace{3em}}c@{\hspace{3em}}c@{\hspace{3em}}c@{\hspace{3em}}c@{\hspace{3em}}c@{\hspace{1em}}}
    \toprule
        Method & AP & AP$_{50}$ & AP$_{75}$ & AP$_{S}$ & AP$_{M}$ & AP$_{L}$ \\
    \midrule
        RN101 & 38.63 & 60.45 & 41.28 & 19.48 & 41.33 & 55.29 \\
        RN18 & 31.25 & 51.07 & 33.10 & 14.18 & 32.80 & 45.53 \\
        ReviewKD \cite{chen2021distilling} & 33.62 & 53.91 & 35.96 & 15.03 & 35.31 & 50.30 \\
        \textbf{CKD+ReviewKD} & \textbf{34.59} & \textbf{55.01} & \textbf{36.82} & \textbf{15.68} & \textbf{36.35} & \textbf{51.45} \\
        $\mathrm{\Delta}$ & \textcolor[RGB]{0, 166, 79}{+0.97} & \textcolor[RGB]{0, 166, 79}{+1.10} & \textcolor[RGB]{0, 166, 79}{+0.86} & \textcolor[RGB]{0, 166, 79}{+0.65} & \textcolor[RGB]{0, 166, 79}{+1.04} & \textcolor[RGB]{0, 166, 79}{+1.15} \\
    \midrule
        RN101 & 38.63 & 60.45 & 41.28 & 19.48 & 41.33 & 55.29 \\
        RN50 & 35.24 & 56.32 & 37.49 & 15.03 & 35.31 & 50.30 \\
        ReviewKD \cite{chen2021distilling} & 36.98 & 58.13 & 39.60 & 17.54 & 39.57 & 53.19 \\
        \textbf{CKD+ReviewKD} & \textbf{37.89} & \textbf{59.45} & \textbf{40.61} & \textbf{18.40} & \textbf{40.62} & \textbf{54.91} \\
        $\mathrm{\Delta}$ & \textcolor[RGB]{0, 166, 79}{+0.91} & \textcolor[RGB]{0, 166, 79}{+1.32} & \textcolor[RGB]{0, 166, 79}{+1.01} & \textcolor[RGB]{0, 166, 79}{+0.86} & \textcolor[RGB]{0, 166, 79}{+1.05} & \textcolor[RGB]{0, 166, 79}{+1.72} \\
    \midrule
        RN50 & 35.24 & 56.32 & 37.49 & 15.03 & 35.31 & 50.30 \\
        MN-V2 & 28.37 & 47.19 & 29.95 & 12.09 & 29.01 & 41.70 \\
        ReviewKD \cite{chen2021distilling} & 31.56 & 50.70 & 33.44 & 12.76 & 32.44 & 47.39 \\
        \textbf{CKD+ReviewKD} & \textbf{32.80} & \textbf{52.64} & \textbf{34.96} & \textbf{14.67} & \textbf{34.38} & \textbf{48.84} \\
        $\mathrm{\Delta}$ & \textcolor[RGB]{0, 166, 79}{+1.24} & \textcolor[RGB]{0, 166, 79}{+1.94} & \textcolor[RGB]{0, 166, 79}{+1.52} & \textcolor[RGB]{0, 166, 79}{+1.91} & \textcolor[RGB]{0, 166, 79}{+1.94} & \textcolor[RGB]{0, 166, 79}{+1.45} \\
    \bottomrule
    \end{tabular}
    \label{tab_segment}
\end{table*}

For classic methods with $n$ samples, a total of $\mathcal{O}(nm)$ triples should be selected for each sample,
\begin{equation}
\binom{m-1}{1}\binom{n-m}{1}=nm-m^2-n+m\propto \mathcal{O}(nm)
\label{eq21}
\end{equation}
\Cref{eq21} shows that for each sample, these methods randomly choose one sample from the same class, including $m-1$ selections, and one sample from a different class, yielding a total of $\mathcal{O}(nm)$ triples. In contrast, CRD only needs to select one sample from different classes, as the positive pair is derived from the same sample,
\begin{align}
\binom{1}{1}\binom{n-m}{1}=n-m\propto \mathcal{O}(n).
\label{eq22}
\end{align}
However, the negative space $\mathcal{O}(n)$ is extensive, requiring a memory bank of size $K$ to store historical features. 
There are $K$ triples available for training each sample. In contrast, the proposed method requires sampling only $\mathcal{O}(c)$ triples,
\begin{align}
\binom{1}{1}\binom{c}{1}=c\propto \mathcal{O}(1),
\label{eq23}
\end{align}
where $c \ll n$ . We randomly shuffle data samples and select $c$ non-overlapping samples sequentially. Therefore, our sample-wise formulation exhibits superior training efficiency compared to the class-wise contrastive formulation.

Finally, we show the training procedure in \textbf{\Cref{alg1}}.

\section{Experiments}

\subsection{Image Classification}
\subsubsection{Datasets} 
We conduct experiments on three benchmark datasets, including CIFAR-100 \cite{Krizhevsky2009LearningML}, ImageNet-1K \cite{Deng2009ImageNet} and Places365 \cite{zhou2017places}. 
The CIFAR-100 dataset contains 100 fine-grained classes, with each class consisting of 600 images. The dataset is divided into 50,000 training images and 10,000 validation images.
The ImageNet-1K dataset, a large-scale classification dataset,  consists of 1000 classes, with approximately 1.28 million training images and 50,000 validation images. 
The Places365 dataset \cite{zhou2017places} is a large-scale scene recognition dataset containing 365 scene categories. This dataset includes 1.8 million training images and 36,500 validation images.

\subsubsection{Experimental Setup} 
We evaluate eleven isomorphic or heterogeneous architectures on CIFAR-100 and comparison methods include vanilla KD \cite{hinton2015distilling}, AT \cite{Zagoruyko2016PayingMA}, FitNet \cite{romero2014fitnets}, RKD \cite{park2019relational}, CRD \cite{tian2020representation}, OFD \cite{heo2019comprehensive}, ReviewKD \cite{chen2021distilling}, DML\cite{zhang2018deep}, TAKD\cite{mirzadeh2020improved}, and DKD \cite{zhao2022decoupled}.
For the large-scale ImageNet-1K and Places365 datasets, we focus on representative architecture pairs: ResNet34 \& ResNet18 and ResNet50 \& MobileNetV1. We compares against vanilla KD \cite{hinton2015distilling}, AT \cite{Zagoruyko2016PayingMA}, OFD \cite{heo2019comprehensive}, CRD \cite{tian2020representation}, ReviewKD \cite{chen2021distilling}, and DKD \cite{zhao2022decoupled}, reporting both top-1 and top-5 recognition rates across all experiments.

\subsubsection{Implementation Details}
All experiments are conducted on NVIDIA RTX 3090 GPUs with dataset-specific settings. 
For CIFAR-100, we set $\alpha$ to 100 and employ a cosine annealing learning rate scheduler. 
For homogeneous and heterogeneous networks, we set batch sizes to 64 and 32, with an initial learning rate of 0.05 and 0.01. 
For ImageNet-1K, we use cosine annealing with an initial learning rate of 0.2.
We set the parameter $\alpha$ to 10 and the batch size to 512.
For Places365,  we set the parameter $\alpha$ to 10 and the batch size to 256.

\begin{figure}[t]
    \centering
    \includegraphics[width=0.85\columnwidth]{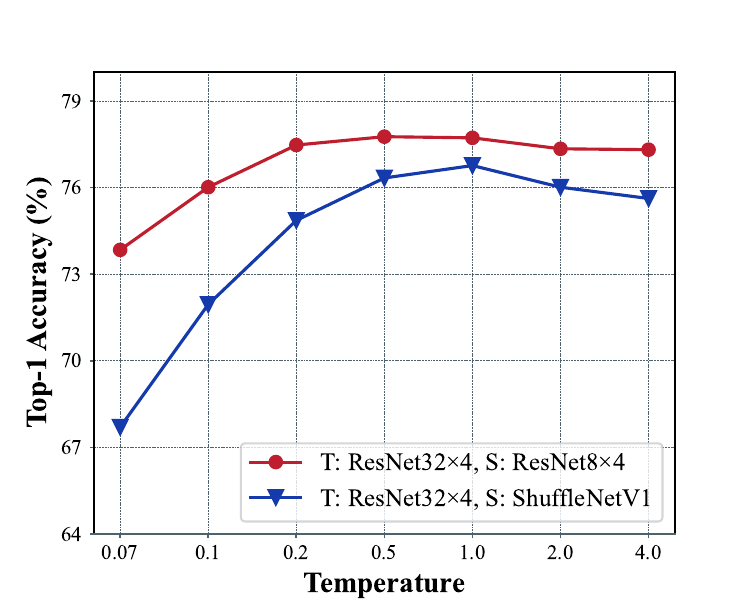}
    \caption{Comparison of temperatures on CIFAR-100. Experiments are conducted on (1) ResNet8$\times$4 and ResNet32$\times$4, (2) ShuffleNetV1 and ResNet32$\times$4.}
    \label{fig7}
\end{figure}

\begin{table}[t]
    \centering
    \caption{Ablation studies of contrastive formulation on the CIFAR-100 dataset. Experiments are conducted on homogeneous and heterogeneous teacher-student pairs. 
    SupCon and CRD are class-wise methods, while our CKD is a sample-wise approach.}
    \begin{tabular}{l|c@{\hspace{0.75em}}c@{\hspace{0.75em}}c|c@{\hspace{0.75em}}c@{\hspace{0.75em}}c}
    \toprule
         Method & \multicolumn{3}{c|}{Homogeneous} & \multicolumn{3}{c}{Heterogeneous} \\ 
    \midrule
        \multirow{2}{*}{Teacher} & RN56 & W-40-2 & W-40-2 & VGG13 & RN50 & RN32$\times$4 \\
          & 72.34 & 75.61 & 75.61 & 74.64 & 79.34 & 79.42 \\
        \multirow{2}{*}{Student} & RN20 & W-16-2 & W-40-1 & MN-V2 &  MN-V2 & SN-V2 \\
          & 69.06 & 73.26 & 71.98& 64.60  & 64.60 & 71.82 \\
    \midrule
        SupCon\cite{khosla2020supervised} & 69.93 & 75.45 & 74.10 & 67.39  & 67.72 & 75.20 \\
        CRD\cite{tian2020representation} & 71.16 & 75.48 & 74.14  & 69.73 & 69.11 & 75.65 \\
        \textbf{CKD} & \textbf{72.12} & \textbf{76.28} & \textbf{75.14} & \textbf{70.11} & \textbf{70.62} & \textbf{78.18} \\
    \bottomrule
    \end{tabular}
    \label{tab6}    
\end{table}

\begin{figure}[t]
\centering
\includegraphics[width=0.85\columnwidth]{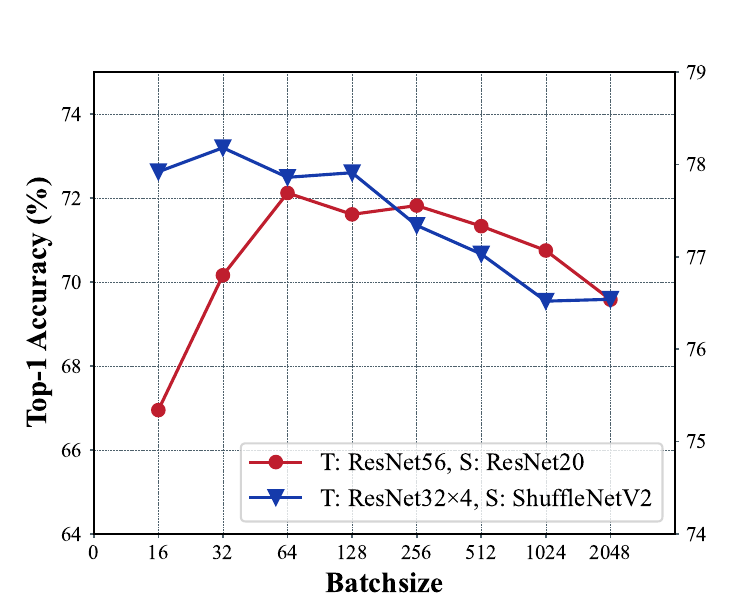}
\caption{Comparison of batch sizes on CIFAR-100. Experiments are conducted on (1) ResNet8$\times$4 and ResNet32$\times$4, (2) ShuffleNetV1 and ResNet32$\times$4.}
\label{fig8}
\end{figure}

\begin{table}[t]
    \centering
    \caption{Ablation studies of triple design on the CIFAR-100 dataset. Experiments are conducted on both homogeneous and heterogeneous teacher-student pairs. 
    Due to the fixed teacher, triples $\left (\mathbf{t}_i, \mathbf{s}_i, \mathbf{t}_j\right)$ are meaningless and excluded from this table.
    }
    \begin{tabular}{l|c@{\hspace{0.75em}}c@{\hspace{0.75em}}c|c@{\hspace{0.75em}}c@{\hspace{0.75em}}c}
    \toprule
         Triple& \multicolumn{3}{c|}{Homogeneous} & \multicolumn{3}{c}{Heterogeneous} \\  
    \midrule
        \multirow{2}{*}{Teacher} & RN56 & W-40-2 & W-40-2 & VGG13 & RN50 & RN32$\times$4 \\
          & 72.34 & 75.61 & 75.61 & 74.64 & 79.34 & 79.42 \\
        \multirow{2}{*}{Student} & RN20 & W-16-2 & W-40-1 & MN-V2 & MN-V2 & SN-V2 \\
          & 69.06 & 73.26 & 71.98 & 64.60 & 64.60 & 71.82 \\
    \midrule
        $\left ({s}_i, {t}_i, {s}_j\right)$ & 71.76 & 75.31 & 74.40 & 70.04 & 68.93 & 78.02 \\
        $\left ({s}_i, {t}_i, {t}_j\right)$ & 71.69 & 75.88 & 74.43 & 70.09 & 69.68 & 77.06 \\
        $\boldsymbol{\left({t}_i, {s}_i, {s}_j\right)}$ & \textbf{72.12} & \textbf{76.28} & \textbf{75.14} & \textbf{70.11} & \textbf{70.62} & \textbf{78.18} \\
    \bottomrule
    \end{tabular}
    \label{tab7}    
\end{table}

\subsubsection{Performance on CIFAR-100} 
We conduct classification experiments by using eleven combinations of isomorphic or heterogeneous teacher-student networks, including ResNet \cite{He2015DeepRL}, WideResNet \cite{Zagoruyko2016WideRN}, VGG \cite{Simonyan2014VeryDC}, ShuffleNet \cite{Zhang2017ShuffleNetAE, Ma2018ShuffleNetVP}, and MobileNet \cite{Sandler2018MobileNetV2IR}. 

\Cref{tab1} shows classification performance using both homogeneous and heterogeneous teacher-student architectures. 
Our method achieves competitive performance against feature-based and logit-based methods.
The proposed CKD achieves an average improvement of 2.20\% over vanilla KD across all 11 settings, with a maximum gain of 3.73\%.
Our method surpasses state-of-the-art feature distillation approaches, with an average accuracy improvement of 1.00\% against CKD. 
Moreover, CKD shows stronger improvements in heterogeneous settings compared to homogeneous configurations.
The superior performance of sample-wise contrastive learning over class-wise alternatives is evident in heterogeneous scenarios.
The sample-wise contrastive mechanism appears particularly effective when dealing with heterogeneous architectures. Unlike feature-based methods requiring explicit layer alignment, CKD creates architecture-agnostic embeddings with logit alignment, better bridging the representation gap between different networks.

\subsubsection{Performance on ImageNet-1K} 
To further verify the effectiveness of CKD on a large-scale dataset, we perform experiments on ImageNet-1K. 
As listed in \Cref{tab2}, our method achieves superior performance with both teacher-student architectures, providing further evidence of the effectiveness on large-scale datasets. 
Specifically, when distilling between homogeneous networks, CKD achieves 72.24\% Top-1 accuracy, representing a 1.58\% improvement over vanilla KD. More remarkably, in the cross-architecture setting, CKD attains 72.97\% Top-1 accuracy with a substantial 4.39\% gain. 
These results not only validate its compatibility with diverse network structures but also establish new state-of-the-art performance on large-scale knowledge distillation.

\subsubsection{Performance on Places365} To evaluate our method on large-scale scene recognition, we conduct experiments on the Places365 validation set. As shown in \Cref{tab3}, our method achieves the SOTA Top-1 accuracy and competitive Top-5 accuracy, demonstrating its effectiveness in complex scene understanding tasks. Notably, our method outperforms CRD in all settings, highlighting the superiority of sample-wise contrastive learning over class-wise approaches. These results validate the robustness and scalability of our method in large-scale knowledge distillation tasks.

\subsection{Object Detection and Instance Segmentation}
\subsubsection{Datasets}
We evaluate our method on the MS-COCO dataset \cite{Lin2014MicrosoftCC}, which contains 118,000 training and 5,000 validation images across 80 object categories. 

\begin{figure}[t] 
\centering 
\includegraphics[width=0.9\columnwidth]{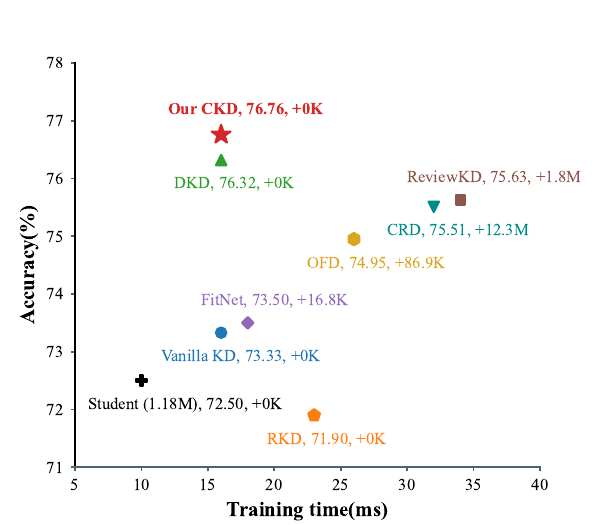} 
\vspace{-0.2cm}
\caption{Comparison of training time per batch and accuracy on CIFAR-100. The teacher model is ResNet32$\times$4, and the student model is ResNet8$\times$4. Additional parameters are denoted after $+$ and are added based on the student's foundation.
} 
\label{fig9} 
\end{figure}

\subsubsection{Experimental Setup} We analyze ResNet-101 \& ResNet-18, ResNet-101 \& ResNet-50, and ResNet-50 \& MobileNetV2 teacher-student pairs. Following standard evaluation protocols, we implement the detection framework using Detectron2.

\subsubsection{Implementation Details}
All experiments are conducted on NVIDIA RTX 3090 GPUs. 
For object detection, we employ Faster R-CNN with FPN as the feature extractor and implement CKD on the R-CNN head. 
For instance segmentation, we utilize Mask R-CNN as the baseline architecture. 
We set $\alpha$ to 0.1 and the temperature $\tau$ to 1.0. 
All training protocols, including learning rate schedulers and loss functions, strictly follow Detectron2.

\subsubsection{Object Detection Performance on MS-COCO} 
\Cref{tab5} presents the object detection results on the MS-COCO dataset under three teacher-student settings. 
We can see that the proposed CKD achieves comparable performance against logit-based approaches but inferior performance compared to feature-based approaches, particularly for ReviewKD. This can be attributed to the fact that object detection heavily relies on spatial representations to locate objects. However, logits lack detailed spatial information \cite{Li2017MimickingVE, wang2019distilling}. ReviewKD \cite{chen2021distilling} conducts multi-level feature distillation, leading to superior performance. To leverage spatial information, we combine ReviewKD with logit-based approaches. Notably, the combined approaches, including CKD+ReviewKD and DKD+ReviewKD, significantly improve detection results compared to their standalone counterparts. 

\subsubsection{Instance Segmentation Performance on MS-COCO}
\Cref{tab_segment} presents instance segmentation results on the MS-COCO dataset across different frameworks. 
Notably, existing knowledge distillation methods specifically designed for instance segmentation remain scarce, which limits comparisons primarily to ReviewKD and the proposed CKD+ReviewKD. 
CKD+ReviewKD achieves performance gains over ReviewKD, with absolute accuracy increases of +0.97, +0.91, and +1.24 points in different teacher-student pairs, respectively. 
The most significant improvement emerges in the RN50-MobileNetV2 setting, particularly for small objects.
The synergy between logit-based distillation CKD and feature alignment ReviewKD addresses both global semantic preservation and local spatial detail transfer. This dual mechanism explains the balanced improvements across all object scales \cite{Li2017MimickingVE}. Experiments demonstrate robustness across different network gaps from RN101-RN50 to RN50-MN-V2, suggesting wide applicability in practical scenarios where compact models are essential.

\subsection{Ablation Study}
\label{ablation studies}
\subsubsection{Effect of Temperature $\tau$}
In the training phase, a high $\tau$ yields a smooth distribution across classes, making distillation more challenging. But a low $\tau$ results in a sharp distribution. To investigate the effect of $\tau$, we conduct ablation studies by using ResNet32$\times4$ \& ResNet8$\times4$ and ResNet32$\times4$ \& ShuffleNetV1 pairs. Typically, $\tau$ is set to a low value within a range of 0.07 to 4.0, while other parameters are kept constant.

\begin{table}[t]
    \centering
    \caption{Classification accuracy (\%) on the CIFAR-100 validation set. WRN-16-2 is used as a student. In the left-hand section, WRN series networks serve as teachers, while in the right-hand section, different series networks are utilized as teachers.}
    \begin{tabular}{l@{\hspace{1.2em}}|c@{\hspace{1.1em}}c@{\hspace{1.1em}}c|c@{\hspace{1.1em}}c@{\hspace{1.1em}}c}
    \toprule
         Method & \multicolumn{3}{c|}{WRN Series} & \multicolumn{3}{c}{Different Series} \\
    \midrule
        \multirow{2}{*}{Teacher} & W-40-2 & W-16-4 & W-28-4 & VGG13 & W-16-4 & RN50 \\
          & 75.61 & 77.21 & 79.12 & 74.64 & 77.21 & 79.34 \\
        Studnet & 73.26 & 73.26 & 73.26 & 73.26 & 73.26 & 73.26\\
    \midrule
         KD\cite{hinton2015distilling} & 74.92 & 76.00 & 75.14 & 74.30 & 76.00 & 74.62 \\
         CKD & \textbf{76.28} & \textbf{76.22} & \textbf{76.42} & \textbf{75.14} & \textbf{76.22} & \textbf{76.45} \\
    \bottomrule
    \end{tabular}
    \label{tab8}    
\end{table}

\Cref{fig7} illustrates the temperature sensitivity analysis for distillation.
Both settings exhibit a clear performance parabola with peak Top-1 accuracy 79.0\% and 76.3\%, respectively, at optimal temperature $\tau=1.0$. 
Our method achieves optimal performance when temperature $\tau$ is approximately 1.0 and performs poorly when $\tau$ is low or high. This suggests that proper temperature scaling balances the sharpness of knowledge transfer.

From \Cref{throrm1}, we conclude that the gradient $\nabla_{\boldsymbol{s}_*}\mathcal{L}_i$ is proportional to $g_i$. When $\tau<1$, the difference between $e^{ f\left ( \boldsymbol{t}_i,\boldsymbol{s}_i \right)/\tau}$ and $e^{ f\left ( \boldsymbol{t}_i,\boldsymbol{s}_j \right)/\tau}$ is exponentially magnified. 
In the limit, we obtain $\lim_{\tau \to 0} g_i= 0$. As a result, $g_i$ approaches zero, causing a vanishing gradient. When $\tau>1$, the difference is reduced and $g_i$ becomes more sensitive to negative samples, as $\lim_{\tau \to \infty} g_i= \frac{B}{B+1}$.

\begin{figure*}[t] 
\centering 
\includegraphics[width=1.85\columnwidth]{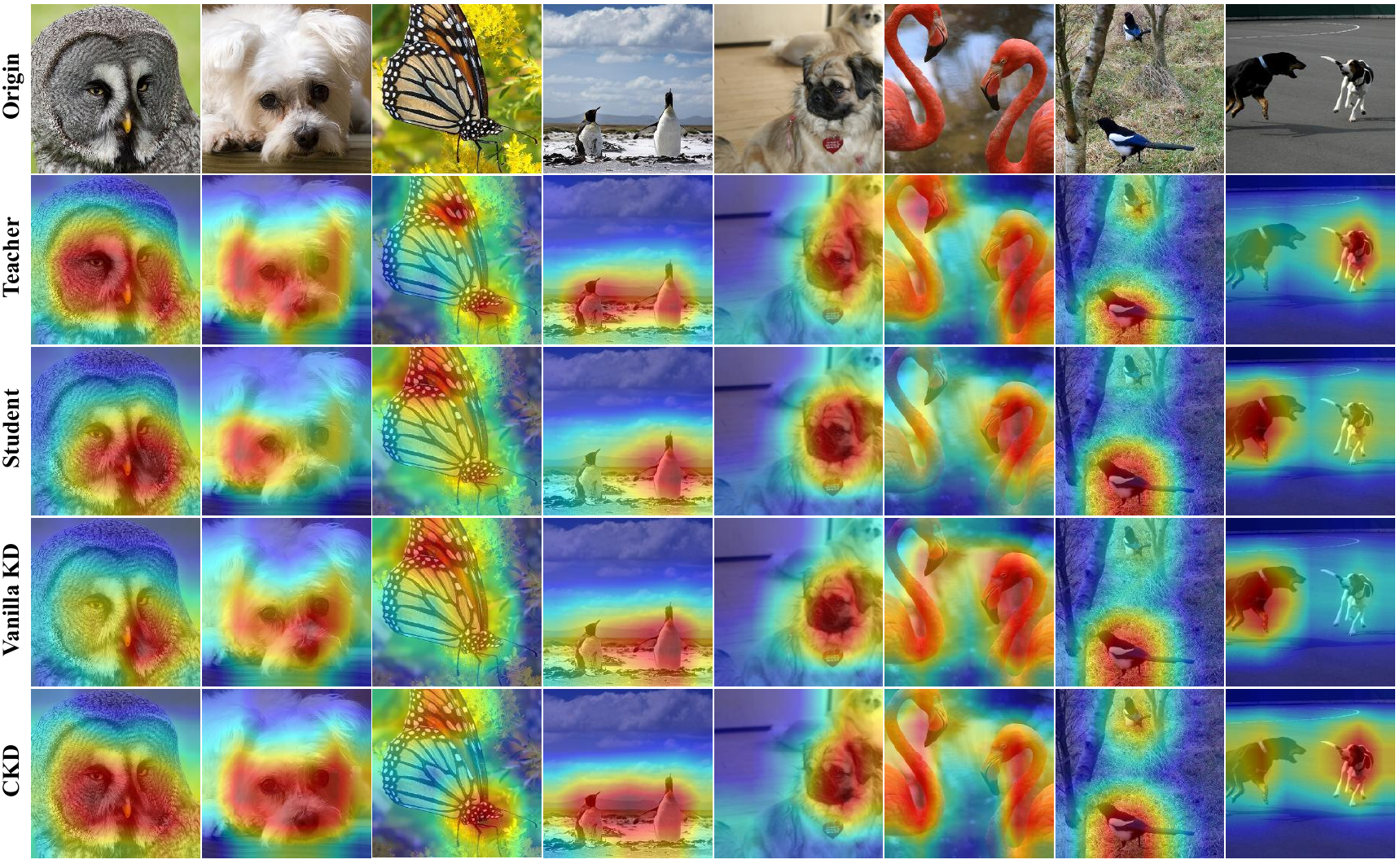} 
\vspace{-0.2cm}
\caption{Class Activation Map via GradCAM. The teacher networks (ResNet34) and the student networks (ResNet18) are trained on ImageNet-1K.} 
\label{fig10} 
\end{figure*}

\subsubsection{Effect of Batchsize}
To analyze the effect of batch size, we utilize two teacher-student pairs: ResNet56 \& ResNet20 and ResNet32$\times4$ \& ShuffleNetV2. We set the batch size $B$ to $2^k$, where $k\in [4,11] \cap  k\in \mathbb{N}$.

\Cref{fig8} presents top-1 accuray across various batch sizes. 
We can see that our method achieves maximum accuracies when the batch size is 64 for ResNet56 \& ResNet20 and 32 for ResNet32$\times4$ \& ShuffleNetV2. 
Batch sizes outside these optimal values yield suboptimal results.

Experiments reveal that a large batch size potentially hampers the efficacy of our method by overwhelming inter-sample distillation through an increased $\beta$. However, a small batch size suffers from insufficient negative pairs due to vanishing gradients. 
The empirical optima should satisfy the category count constraint $B<c$. When $B > c$, distillation focuses on inter-sample relations, degrading class-specific knowledge transfer. This analysis proves that batch size significantly balances intra/inter-sample knowledge transfer.

\subsubsection{Influence of Contrastive Formulation}
\Cref{fig6} visualizes differences between class-wise and sample-wise contrastive learning frameworks. In class-wise approaches such as SupCon \cite{khosla2020supervised}, positive pairs are formed exclusively within the same class. 
This category also includes methods like CRD \cite{tian2020representation}, which specifically constructs positive pairs from teacher-student models of the same sample while selecting negative pairs from different classes. By contrast, the sample-wise CKD adopts a different strategy in that it selects negative pairs from distinct samples, even when they potentially belong to the same class.

\Cref{tab6} compares classification performance across three contrastive learning frameworks.
Notably, CRD exhibits improved accuracy over SupCon by maintaining a memory bank to access negative instances.
However, our CKD achieves superior performance by employing sample-wise negative pairs, which is particularly evident in heterogeneous settings where it outperforms CRD by 2.53\% with the SN-v2 setting. 

\subsubsection{Influence of Triple Design}
In knowledge distillation, triplet elements can be sampled from either the teacher (t) or student (s) models. Hence, there are four kinds of triples, including $\left(\boldsymbol{s}_i, \boldsymbol{t}_i, \boldsymbol{s}_j\right)$, $\left(\boldsymbol{s}_i, \boldsymbol{t}_i, \boldsymbol{t}_j\right)$, $\left(\boldsymbol{t}_i, \boldsymbol{s}_i, \boldsymbol{t}_j\right)$ and $\left(\boldsymbol{t}_i, \boldsymbol{s}_i, \boldsymbol{s}_j\right)$. $\left(\boldsymbol{s}_i, \boldsymbol{t}_i, \boldsymbol{s}_j\right)$ means that we align $\mathbf{s}_i$ and $\boldsymbol{t}_i$ while distancing $\boldsymbol{s}_i$ from $\boldsymbol{s}_j$. This triple is implemented through diagonal modification of the affinity matrix between $\boldsymbol{s}_i$ and $\boldsymbol{t}_i$ pairs. Notably, $\left(\boldsymbol{t}_i, \boldsymbol{s}_i, \boldsymbol{t}_j\right)$ reduces to intra-sample alignment since teacher logits remain fixed during distillation.. 
 
\Cref{tab7} shows the classification performance for three triple settings.
The teacher-anchored triples $\left(\boldsymbol{t}_i, \boldsymbol{s}_i, \boldsymbol{s}_j\right)$ demonstrates superior performance across all settings, achieving 1.43\% higher accuracy than $\left(\boldsymbol{s}_i, \boldsymbol{t}_i, \boldsymbol{t}_j\right)$ in SN-V2 heterogeneous architecture. This advantage stems from its clear optimization objective by anchoring on fixed teacher features $\boldsymbol{t}_i$ to simultaneously attract student $\boldsymbol{s}_i$ and repel $\boldsymbol{s}_j$.
However, as discussed in \Cref{subsec:ckd}, triples $\left(\boldsymbol{s}_i, \boldsymbol{t}_i, \boldsymbol{t}_j\right)$ risks gradient conflict due to competing objectives between moving $\boldsymbol{s}_i$ toward $\boldsymbol{t}_i$ and away from $\boldsymbol{t}_j$.
For triples $\left(\boldsymbol{s}_i, \boldsymbol{t}_i, \boldsymbol{s}_j\right)$, this strategy mainly performs contrastive learning in the student logits space, yet insufficiently leverages the prior knowledge from the teacher model.  

\subsection{More Analyses}
\subsubsection{Training Efficiency}
To evaluate the efficiency-accuracy trade-off, we compare our method against knowledge distillation baselines on CIFAR-100. 
As shown in \Cref{fig9}, our method exhibits high training efficiency, striking a good balance between model parameters and classification performance. Feature-based methods require more parameters and longer training times than logit-based methods. ReviewKD and CRD, in particular, take twice the training time over our method due to multi-level feature review and memory bank. Our method simultaneously improves accuracy and reduces training time.

\subsubsection{Effect of Teacher Capacity}
Existing studies \cite{cho2019efficacy, wang2021knowledge} have demonstrated that large teacher models may be suboptimal for distilling knowledge into smaller student models due to parameter and capacity gaps between the two architectures. To investigate this phenomenon, we conduct experiments to examine how teacher networks with varying model capacities affect knowledge transfer efficiency.

\Cref{tab8} compares knowledge distillation performance on the CIFAR-100 dataset. We can see that our method achieves comparable performance with increasing WRN teacher capacity, whereas vanilla KD exhibits variation in performance. 
For vanilla KD, larger models lead to performance degradation, achieving the accuracies of WRN-16-4 and WRN-28-4 at 76.00 and 75.14, respectively. 
In contrast, our CKD achieves progressive accuracy improvements as teacher capacity increases from WRN-40-2 to WRN-28-4, demonstrating effective utilization of enhanced teacher representations.
Wide Residual Networks \cite{Zagoruyko2016WideRN} is a ResNet variant. 
Our method also obtains an accuracy on par with ResNet50 as the teacher model, while vanilla KD has the second-worst performance. 
Both methods show limited effectiveness with VGG13 architecture, suggesting architectural compatibility influences distillation efficacy.
These results prove that our method can partially mitigate the parameter and capacity gaps of larger models.

\subsection{Visualizations}
\Cref{fig10} demonstrates the semantic alignment between CKD and teacher networks through GradCAM visualizations, where the ResNet34-ResNet18 pair was trained on ImageNet-1K. 
We can observe that our method effectively handles complex scenes and multiple objects. 
In multi-object cases, CKD captures both discriminative features of birds circled in red, while the student baseline and vanilla KD focus solely on the foreground bird.
Across diverse scenes like natural landscapes and indoor settings, CKD consistently preserves the spatial attention patterns, particularly in fine-grained regions like butterfly wing patterns.
In the final column, CKD aligns with the teacher model in localizing the target dog, whereas the student network and vanilla KD erroneously emphasize the distractor dog. 
These visual evidence confirm that CKD maintains teacher-aligned sample-specific semantics rather than generic category features.

\section{Conclusion}
In this paper, we have presented a contrastive knowledge distillation approach for image classification and object detection. To be specific, we formulate knowledge distillation as a logit alignment problem, mimicking the intra-sample numerical value and inter-sample semantic structure of teacher models. By converting the structure-preserving alignment into a sample-wise contrastive learning framework, our approach can be effectively and efficiently trained with well-designed positive and negative pairs. Experiments demonstrate that our method achieves highly competitive performance against state-of-the-art methods on benchmarks. In the future, we will explore the application of this framework to large multimodal models.

\bibliographystyle{IEEEtranS}
\bibliography{egbib}

\end{document}